\documentclass{article}

\PassOptionsToPackage{numbers,sort&compress}{natbib}
\usepackage[main,final]{neurips_2026}
\usepackage[utf8]{inputenc} % allow utf-8 input
\usepackage[T1]{fontenc}    % use 8-bit T1 fonts
\usepackage{hyperref}       % hyperlinks
\usepackage{url}            % simple URL typesetting
\usepackage{booktabs}       % professional-quality tables
\usepackage{amsfonts}       % blackboard math symbols
\usepackage{nicefrac}       % compact symbols for 1/2, etc.
\usepackage{microtype}      % microtypography
\usepackage{xcolor}         % colors
\usepackage{amsmath}
\usepackage{amssymb}
\usepackage{graphicx}   % \resizebox, \rotatebox
\usepackage{booktabs}   % \toprule, \midrule, \bottomrule, \cmidrule
\usepackage{multirow}   % \multirow
\usepackage{makecell}   % \makecell
\usepackage[table]{xcolor} % \cellcolor
\usepackage{tabularx}
\usepackage{wrapfig}
\usepackage{adjustbox}
\newcolumntype{Y}{>{\centering\arraybackslash}X}

\usepackage{amsthm}

\newtheorem{theorem}{Theorem}[section]

\theoremstyle{definition}

\newtheorem{assumption}[theorem]{Assumption}

\theoremstyle{remark}

\usepackage{algorithm}
\usepackage{algpseudocode}

\newcommand{\accval}[1]{\fontsize{9.8pt}{10pt}\selectfont #1}

\title{LSVD: Loss-Aware Low-Rank Approximation for Efficient Low-Precision Vision-Language Models}
\author{
Haiyu Wang$^{1}$ \quad
Yutong Wang$^{1}$\thanks{Authors contributed equally; the order of authorship was assigned randomly.} \quad
Leshu Li$^{2}$\footnotemark[1] \quad
Yihui Ren$^{3}$ \quad 
Sai Qian Zhang$^{1,2}$ \\
$^1$Tandon School of Engineering, New York University \\ 
$^2$Courant Institute of Mathematical Sciences, New York University\\ 
$^3$Brookhaven National Laboratory\\
\texttt{\{hw3689, yw6594, ll5914\}@nyu.edu} \quad
\texttt{yren@bnl.gov} \quad
\texttt{sai.zhang@nyu.edu}
}

\begin{document}

\maketitle

\begin{abstract}
Vision-language models (VLMs) deliver strong multimodal reasoning capabilities, but their large computational cost and high parameter counts make deployment challenging on resource-constrained devices. Low-rank decomposition has emerged as a promising compression technique, yet existing methods often optimize local matrix reconstruction error, rely on uniform or heuristic rank allocation, and focus mainly on attention projections while leaving feed-forward networks underexplored. In this paper, we propose~\textit{LSVD}, a loss-aware low-rank approximation framework for efficient low-precision VLMs. LSVD derives a curvature-weighted SVD objective from a second-order approximation of the model loss and uses Kronecker-factored Fisher information to guide decomposition toward downstream performance rather than reconstruction alone. We further introduce a loss-aware cross-layer rank allocation strategy based on calibration gradients, enabling more effective parameter budgeting across layers. Finally, we extend low-rank compression to FFN layers through a hybrid scheme that combines SVD with quantization.
The evaluation results show that LSVD achieves more than $2.3\times$ decoding speedup over previous work while preserving strong accuracy under low-precision inference.

\end{abstract}

\section{Introduction}
\begin{wrapfigure}{r}{0.7\textwidth}
    \centering
    \vspace{-10pt}
\includegraphics[width=0.7\textwidth]{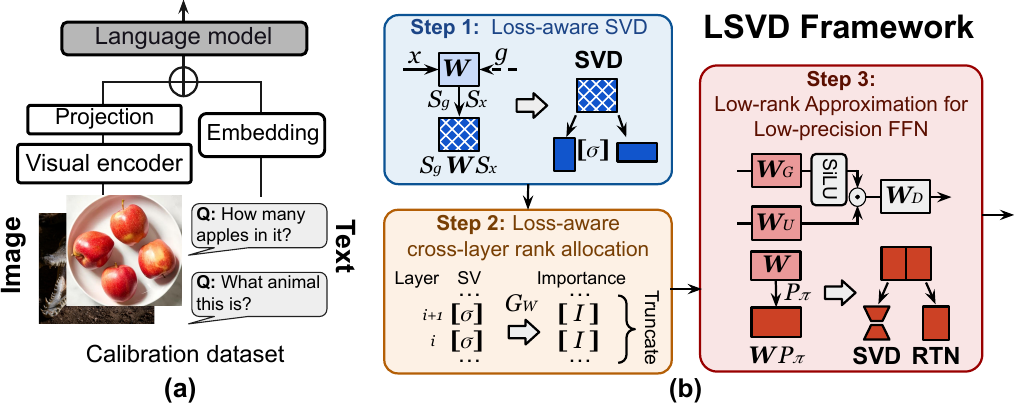}
    \vspace{-10pt}
    \caption{
    (a) A VLM model. (b) Overview of LSVD Framework.
    }
    \label{fig:lsvd_overall}
    \vspace{-5pt}
\end{wrapfigure}
Vision–language models (VLMs) have become an important research direction in modern AI because they bridge visual perception and language understanding. By reasoning over both image and text inputs, VLMs support a wide range of multimodal tasks, such as image captioning~\cite{hu2022scaling, li2022blip,li2023blip2, chen2022visualgpt, dzabraev2024vlrm,xiao2026staying}, visual question answering~\cite{chappuis2022prompt, bazi2023vision, bai2023qwen,wang2024surgical,xiao2026layer,yu2025core3d}, and multimodal semantic retrieval~\cite{li2024searchlvlms,sun2025leveraging,yu2026elsa3d,shen2026fine}. Despite their strong performance, these models impose substantial implementation costs. Processing high-dimensional visual features together with textual context leads to intensive computation, while autoregressive decoding repeatedly accesses model states and generated tokens, creating severe memory-bandwidth pressure during inference and therefore high inference latency. On the other hand, the large size of VLMs leads to substantial storage overhead, further limiting their deployment on storage-constrained devices.

While quantization and pruning have been widely explored to reduce the computation and storage costs of large models~\cite{lin2024duquant,tseng2024quip,ashkboos2024quarot,xiang2024dfrot,wang2024qvlm}, low-rank decomposition has recently attracted increasing attention as another promising compression approach. By factorizing the query (Q), key (K), and value (V) projection matrices in self-attention blocks into low-rank components, prior studies have demonstrated notable reductions in parameter count and computational cost~\cite{wang2025svdllmv2,yuan2023asvd,wang2024svd,li2025adasvd,li2024svdqunat,chang2024palu,wang2025dobi,wang2025qsvd,wang2026wsvd,xiao2026not}. However, several important directions remain underexplored. First, although individual optimization techniques have shown promising gains, it remains unclear whether combining them can yield further improvements. Second, rank allocation across VLM layers remains insufficiently studied, although 
different layers may exhibit distinct sensitivity to compression and therefore require non-uniform rank assignments. Finally, most existing efforts focus primarily on attention projections, while the feed-forward network (FFN) layers, which account for a large fraction of model parameters and storage cost, remain a critical but less explored target for low-rank compression.

To address these limitations, we systematically revisit low-rank compression for VLMs along these underexplored directions, and propose~\textit{Loss-Aware Low-Rank Approximation for
Efficient Low-Precision Vision-Language Models} (LSVD). Our contributions are summarized as follows:

\begin{itemize}
% \item LSVD formulates SVD under a curvature-weighted objective derived from a second-order approximation of the model loss, using Kronecker-factored Fisher information to make the decomposition better aligned with downstream VLM performance rather than only minimizing reconstruction error.
\item LSVD derives a loss-aware SVD objective from a second-order approximation of the model loss, using Kronecker-factored Fisher information to align low-rank decomposition with downstream VLM performance rather than only minimizing reconstruction error.
\item LSVD provides an analysis of why singular values obtained from loss-aware SVD are not directly comparable across layers and heads, and introduces a calibration-gradient-based importance score for more reliable global rank allocation.
\item LSVD extends low-rank compression to FFN layers through a hybrid SVD--quantization scheme that assigns each channel to the compression branch it can better tolerate.
\item We develop Triton kernels for the hybrid low-rank FFN and low-precision SA layers. The evaluation results show that LSVD achieves a decoding speedup of $4.7\times$ over Flash Decoding and $2.3\times$ over WSVD, while preserving accuracy under low-precision inference.
\end{itemize}

\section{Related Work}
\label{sec:related_work}
\paragraph{Vision-Language Model}
VLMs extend LLMs to multimodal reasoning by incorporating visual inputs into language-conditioned generation~\cite{li2022blip, li2023blip2, liu2023visual-llava, 10.5555/3666122.3668264, beyer2024paligemma, grattafiori2024llama, wang2024qwen2,xiao2025visual,xiao2026adapting}. By aligning visual features with the LLM semantic space, they support image captioning, visual question answering, grounding, document understanding, and video comprehension. Modern VLMs typically consist of a vision encoder, a projection module, and an autoregressive LLM, as seen in BLIP/InstructBLIP~\cite{li2022blip, li2023blip2}, LLaVA~\cite{liu2023visual-llava}, PaLI-Gemma~\cite{beyer2024paligemma}, Qwen-VL~\cite{wang2024qwen2}, and SmolVLM~\cite{marafioti2025smolvlm}. Despite their strong capabilities, VLMs are expensive to deploy because they inherit LLM computation and memory costs while adding vision encoders and long visual-token sequences. This challenge motivates compact VLMs~\cite{yuan2023tinygpt, zhou2024tinyllava, marafioti2025smolvlm} and compression techniques such as low-rank decomposition and quantization.

%Vision--Language Models (VLMs) extend LLMs to multimodal reasoning by incorporating visual inputs into language-conditioned generation~\cite{li2022blip, li2023blip2, liu2023visual-llava, 10.5555/3666122.3668264, beyer2024paligemma, grattafiori2024llama, wang2024qwen2}. By aligning visual features with the LLM semantic space, VLMs support tasks such as image captioning, visual question answering, grounding, document understanding, and video comprehension. Modern VLMs typically use a modular architecture, where a vision encoder extracts visual tokens, a projection module maps them into the LLM embedding space, and an autoregressive LLM jointly processes visual and textual tokens. This design underlies models such as BLIP/InstructBLIP~\cite{li2022blip, li2023blip2}, LLaVA~\cite{liu2023visual-llava}, PaLI-Gemma~\cite{beyer2024paligemma}, Qwen-VL~\cite{wang2024qwen2}, and SmolVLM~\cite{marafioti2025smolvlm}. Despite their strong capabilities, VLMs remain costly to deploy because they inherit the computation and memory demands of LLMs while adding vision encoders and long visual-token sequences. This makes deployment difficult on edge and real-time systems with limited memory, bandwidth, and latency budgets, motivating compact VLMs~\cite{yuan2023tinygpt, zhou2024tinyllava, marafioti2025smolvlm} and compression techniques such as low-rank decomposition and quantization.

\paragraph{Low-Rank Approximation for Large Models}
%\label{sec:rw-lra}
Low-rank approximation is a classical approach for reducing the storage and computational cost of large matrices. A standard formulation is Singular Value Decomposition (SVD)~\cite{jolliffe2016principal}, which factorizes a weight matrix $W \in \mathbb{R}^{m \times n}$ as $W = U \Sigma V^{T}$, where $U$ and $V$ contain the left and right singular vectors, and $\Sigma$ is a diagonal matrix of singular values. Retaining only the largest $r$ singular components gives the rank-$r$ approximation: $W \approx U_r \Sigma_r V_r^{T}$,
% {\small
% \begin{equation}
% \label{eqn:low_rank_svd}
%     W \approx U_r \Sigma_r V_r^{T}
% \end{equation}}
which can be implemented as two smaller matrices, $W \approx AB$ with $A= U_r \Sigma_r^{1/2}\in \mathbb{R}^{m \times r}$ and $B=\Sigma_r^{1/2}V_r^{T} \in \mathbb{R}^{r \times n}$. When $r \ll \min(m,n)$, this decomposition reduces both parameter count and arithmetic cost.
SVD-based compression has been extensively studied for large language models and vision--language models~\cite{noach2020compressing, hsu2022language, yuan2023asvd, wang2024svd, li2024svdqunat, chang2024palu, li2025adasvd, wang2025dobi, wang2025qsvd}. While early work directly applied vanilla SVD to pretrained weights~\cite{noach2020compressing}, later methods improve robustness by incorporating calibration or sensitivity information: FWSVD~\cite{hsu2022language} uses Fisher information, ASVD~\cite{yuan2023asvd} accounts for activation outliers, and SVD-LLM~\cite{wang2024svd} and Dobi-SVD~\cite{wang2025dobi} reduce layer-wise truncation error. Recent methods further refine rank allocation~\cite{li2025adasvd, wang2025qsvd} or combine low-rank decomposition with quantization~\cite{li2024svdqunat,wang2026wsvd}.

Beyond weight compression, low-rank approximation has also been applied to inference-time memory reduction and acceleration, particularly for KV cache compression~\cite{chang2024palu, yu2024effectively,wang2026wsvd}. By decomposing the key and value projection matrices, for example $W_K \approx A_KB_K$ and $W_V \approx A_VB_V$, the cache can keep lower-dimensional latent representations $C_K = xA_K$ and $C_V = xA_V$, where $x$ denotes the input. The full keys and values are then reconstructed only when needed as $K = C_KB_K$ and $V = C_VB_V$, avoiding storage of the full-dimensional KV states. This reduces KV cache footprint and memory traffic during autoregressive decoding, which can improve decoding throughput, especially in long-context generation and multimodal models with large token budgets.

\paragraph{Quantization for Large Models}
%\label{sec:rw-quant}
Quantization reduces the memory and inference cost of large models by mapping weights and activations to low-bit formats, with post-training quantization (PTQ) being especially attractive because it avoids retraining. Representative LLM PTQ methods include GPTQ~\cite{frantar2022gptq}, AWQ~\cite{lin2024awq}, and SmoothQuant~\cite{xiao2023smoothquant}, which respectively use second-order error compensation, activation-aware channel protection, and channel-wise scaling to improve robustness. Recent methods further improve low-bit quantization through learned parameters, rotations, and optimized kernels, such as OmniQuant~\cite{shao2023omniquant}, QuaRot~\cite{ashkboos2024quarot}, SpinQuant~\cite{liu2024spinquant}, and QServe~\cite{lin2025qserve}. For VLMs, quantization is more challenging due to modality-dependent activation distributions and different sensitivities between visual and textual tokens; Q-VLM~\cite{wang2024qvlm} and MBQ~\cite{li2024mbq} address these issues through cross-layer dependency modeling and modality-balanced calibration.

\paragraph{Kronecker-Factored Approximate Curvature}
%\label{sec:rw-kfac}
Second-order curvature provides a principled way to measure parameter sensitivity, but the full Hessian or Fisher matrix is prohibitively expensive to compute and store for large models. For a linear layer with $d_{\mathrm{in}}d_{\mathrm{out}}$ parameters, the corresponding Fisher block contains $O(d_{\mathrm{in}}^2d_{\mathrm{out}}^2)$ entries. K-FAC~\cite{martens2015optimizing} makes second-order curvature tractable by approximating this block with a Kronecker product of activation and gradient covariance factors, reducing the curvature representation to two much smaller covariance matrices. For a linear layer $y=xW$ with input activation $x$ and output gradient $g$, K-FAC approximates the Fisher information: $F_W \approx \mathbb{E}[g^\top g] \otimes \mathbb{E}[x^\top x]$, 
% \begin{equation}
% F_W \approx \mathbb{E}[g^\top g] \otimes \mathbb{E}[x^\top x]
% \end{equation}
thereby retaining structured curvature information at substantially lower cost.
Originally proposed for natural-gradient optimization~\cite{martens2015optimizing}, Kronecker-factored curvature has also been used for curvature-aware compression. EigenDamage~\cite{wang2019eigendamage} performs structured pruning in the Kronecker-factored eigenbasis, and the LLM Surgeon~\cite{van2023llm} scales such approximations to large language models for pruning and weight-update compensation. In this work, we use K-FAC as a tractable approximation of local loss curvature, enabling low-rank approximation to be guided by estimated model-loss sensitivity rather than only local reconstruction error.

\section{Method}
\label{sec:method}

\subsection{Preliminary}
\label{sec:loss-aware-lra}

Conventional SVD-based low-rank approximation is usually formulated as a local
matrix reconstruction problem. Given a weight matrix $W$, truncated SVD solves
$\arg\min_{\widehat{W}: \mathrm{rank}(\widehat{W})\le r}
\|W-\widehat{W}\|_F^2$, where $\widehat{W}$ denotes the compressed weight with low rank.
Since this objective ignores the input distribution, activation-aware methods~\cite{yuan2023asvd,wang2024svd,wang2025dobi}
instead minimize layer-wise output error on calibration activations, e.g.,
$\|X(W-\widehat{W})\|_F^2$.
Other approaches further incorporate gradients or Fisher to guide
truncation and rank allocation~\cite{hsu2022language,wang2025qsvd,wang2026wsvd}.
Nevertheless, the decomposition itself remains largely driven by local reconstruction criteria rather than the end-to-end training objective. In the following sections, we therefore derive a loss-aware low-rank objective from a local curvature approximation of the model loss as illustrated in Fig.~\ref{fig:loss-aware-svd}. Following K-FAC~\cite{martens2015optimizing} and prior curvature-aware compression methods~\cite{wang2019eigendamage,van2023llm}, we approximate the loss increase caused by compression using a second-order Kronecker-factored Fisher surrogate.
The derivation from a local quadratic loss approximation is provided in Appendix~\ref{sec:appx-loss-surrogate}.

\begin{assumption}[K-FAC loss surrogate for weight compression]
\label{assm:kfac-loss-surrogate}
Consider a linear layer $y=xW$, and let
$\Delta W=\widehat{W}-W$ denote the perturbation introduced by replacing the pretrained weight $W$ with its compressed approximation $\widehat{W}$. The compression-induced loss increase is estimated by:
{\small
\begin{equation}
\Delta \mathcal{L}
\approx
\frac{1}{2}
\mathrm{vec}(\Delta W)^{\top}
(G\otimes X)
\mathrm{vec}(\Delta W)
\label{eq:kfac-loss-surrogate}
\end{equation}
}
where $x$ is the input activation, $g=\nabla_y\ell$ is the output gradient of the per-sample loss, and $X=\mathbb{E}[x^{\top}x]$ and $G=\mathbb{E}[g^{\top}g]$ are the two K-FAC factors estimated on the calibration set.
\end{assumption}

\subsection{Loss-Aware Singular Value Decomposition}
\label{sec:loss-aware-svd} 
\begin{figure}
    \centering
    \includegraphics[width=1\linewidth]{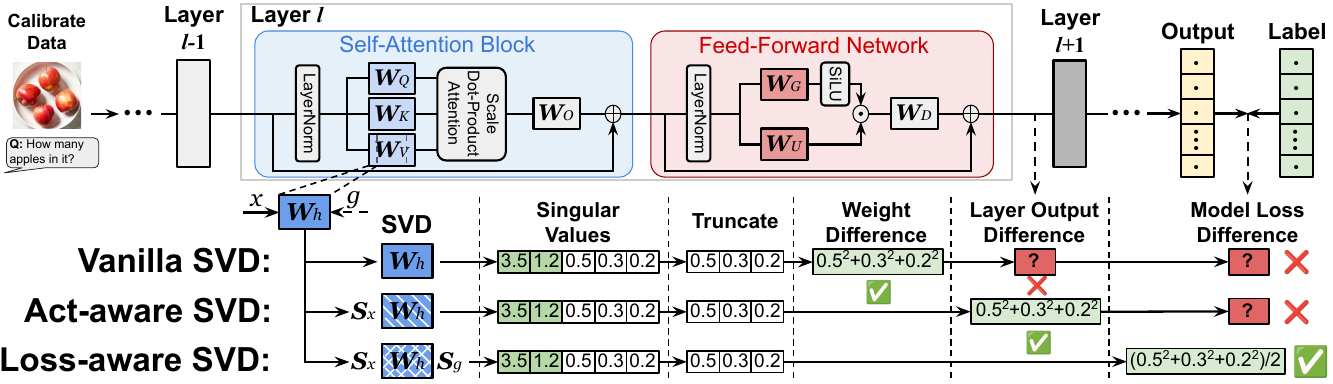}
    \vspace{-15pt}
    \caption{Vanilla SVD minimizes weight reconstruction error. Activation-aware SVD such as SVD-LLM~\cite{wang2024svd} reduces layer-wise output error with input activation statistics. Loss-aware SVD performs SVD in a curvature-weighted weight space to minimize the compression-induced model loss.}
    \label{fig:loss-aware-svd}
    \vspace{-13pt}
\end{figure}
Building on Assumption~\ref{assm:kfac-loss-surrogate}, we derive the low-rank approximation of a weight matrix that minimizes the estimated compression-induced loss increase. Consider a candidate linear layer with weight $W\in\mathbb{R}^{m\times n}$ and its rank-$r$ approximation $\widehat W_r$. The associated activation and gradient covariances are factorized as $X=S_x^{\top}S_x$ and $G=S_gS_g^{\top}$ with Cholesky factorization~\cite{higham1990analysis}, where $S_x$ and $S_g$ are assumed
nonsingular.

\begin{theorem}[Optimal loss-aware rank-$r$ weight approximation]
\label{thm:loss-aware-svd}
Under Assumption~\ref{assm:kfac-loss-surrogate}, define the curvature-weighted
weight matrix $\widetilde W=S_xWS_g$ and let
$\widetilde W=U\Sigma V^{\top}$ be its SVD, where
$\Sigma=\operatorname{diag}(\sigma_1,\ldots,\sigma_q)$ and
$q=\min\{m,n\}$. The rank-$r$ approximation of $W$ that minimizes the estimated compression-induced loss increase, and the corresponding truncation loss, are:
{\small
\begin{equation}
\widehat W_r^{\star}
=
S_x^{-1}U_r\Sigma_rV_r^{\top}S_g^{-1},
\qquad
\Delta\mathcal{L}_r
\approx
\frac{1}{2}\sum_{i=r+1}^{q}\sigma_i^2
\label{eq:loss-aware-svd-solution-compact}
\end{equation}
\vspace{-10pt}
}
\end{theorem}

Theorem~\ref{thm:loss-aware-svd} directly leads to an implementable procedure.
Given a calibration set, we collect the input activations $x$ and output
gradients $g$ for each candidate linear layer, and estimate
$X=\mathbb{E}[x^{\top}x]$ and $G=\mathbb{E}[g^{\top}g]$. Then we compute factors $S_x$ and $S_g$ using Cholesky factorization~\cite{higham1990analysis}, 
%either Cholesky factorization~\cite{higham1990analysis} or eigen-decomposition~\cite{abdi2007eigen},
such that $X=S_x^{\top}S_x$ and $G=S_gS_g^{\top}$. Next, we form the transformed weight $\widetilde W=S_xWS_g$, perform truncated SVD to obtain $\widetilde W_r^{\star}=U_r\Sigma_rV_r^{\top}$, and map it back as:
{\small
\begin{equation}
\widehat W_r^{\star}
=
S_x^{-1}U_r\Sigma_rV_r^{\top}S_g^{-1}
=
A_rB_r,
\qquad
A_r=S_x^{-1}U_r\Sigma_r^{1/2},
\quad
B_r=\Sigma_r^{1/2}V_r^{\top}S_g^{-1}
\label{eq:loss-aware-svd-solution}
\end{equation}
\vspace{-10pt}
}

The corresponding truncation loss is:
{\small
$
\Delta\mathcal{L}_r
\approx
\frac{1}{2}
\left\|S_x(W-\widehat W_r^{\star})S_g\right\|_F^2
=
\frac{1}{2}\sum_{i=r+1}^{q}\sigma_i^2
\label{eq:loss-aware-tail-loss}
$}.
At inference, we do not materialize
$\widehat W_r^{\star}\in\mathbb{R}^{m\times n}$; instead, we replace $xW$ with
$xA_rB_r$, where $A_r\in\mathbb{R}^{m\times r}$ and
$B_r\in\mathbb{R}^{r\times n}$. This reduces the parameter count and dominant multiply-accumulation (MAC) cost from $mn$ to $r(m+n)$, giving the compression ratio $\rho_1=\frac{r(m+n)}{mn}$. 
% {\small
% \begin{equation}
% \rho_1
% =
% \frac{r(m+n)}{mn}
% \label{eq:svd-param-mac-ratio}
% \end{equation}
% }
Thus, when $\rho_1<1$, or equivalently $r<mn/(m+n)$, the low-rank form reduces parameter count and computational cost.

\subsection{Loss-Aware Cross-Layer Rank Allocation}
\label{sec:rank-alloc}

Theorem~\ref{thm:loss-aware-svd} shows that the squared singular values of the curvature-weighted matrix estimate the loss increase caused by truncating the corresponding singular components. For layer $\ell$, let $\sigma_{\ell,i}$ denote the $i$-th singular value after loss-aware SVD. This gives a natural intuition for rank allocation: singular components with larger $\sigma_{\ell,i}^2$ should be preserved first, and
layers or heads with more large singular values should receive more ranks.
However, using these K-FAC-based singular values for global rank allocation assumes that values computed from different layers or attention heads are directly comparable. In other words, a larger $\sigma_{\ell,i}$ in one layer $\ell$ should correspond to a
larger loss increase than a smaller $\sigma_{\ell',j}$ in another layer $\ell'$. This comparability is not guaranteed in practice. K-FAC estimates curvature separately for each layer or attention head, so the loss estimate in Assumption~\ref{assm:kfac-loss-surrogate} may differ from the empirical-Fisher estimate by layer- or head-dependent scaling factors.
In this case, the same value of $\sigma_{\ell,i}$ may correspond to different loss increases in different layers or heads. Although $\sigma_{\ell,i}$ is meaningful for local truncation within the same layer or head, it may be unreliable for ranking components globally across layers and heads. 
% To examine this issue, we compare the difference of average curvature magnitude $\eta_\ell$ given by  empirical Fisher with that given by its K-FAC approximation \textcolor{blue}{for layer $\ell$} 
To examine this issue, we use the trace ratio $\eta_\ell$ between the empirical-Fisher block and its K-FAC approximation for layer $\ell$ to indicate K-FAC's estimation bias:
%(\textcolor{red}{what is $\eta_\ell, F_\ell, G_\ell, X_\ell, G_\ell$? Have you ever defined them? Why Tr(.) is used? Is it used to estimated empirical fisher score?})
{\small
\begin{equation}
\eta_\ell
\triangleq
\frac{\operatorname{Tr}(F_\ell)}
{\operatorname{Tr}(K_\ell)}
=
\frac{\operatorname{Tr}(F_\ell)}
{\operatorname{Tr}(G_\ell\otimes X_\ell)}
=
\frac{\operatorname{Tr}(F_\ell)}
{\operatorname{Tr}(G_\ell)\operatorname{Tr}(X_\ell)}
\label{eq:eta-scale-ratio}
\end{equation}
}

\begin{wrapfigure}{r}{0.25\textwidth}
    \centering
    \vspace{-20pt}
    \includegraphics[width=0.25\textwidth]{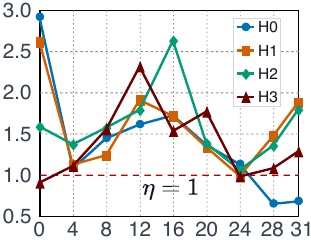}
    \vspace{-15pt}
    \caption{Values of $\eta_{\ell,h}$ across 32 layers (4 heads illustrated).}
    \label{fig:eta-diagnostic}
    \vspace{-15pt}
\end{wrapfigure}
where $F_\ell$ is the empirical-Fisher block of layer $\ell$, $K_\ell=G_\ell\otimes X_\ell$ is its K-FAC approximation, and $X_\ell$ and $G_\ell$ are the activation-side and gradient-side K-FAC factors of layer $\ell$. We provide a formal derivation of this trace ratio and its connection to the
K-FAC approximation error in Appendix~\ref{sec:appx-scale-diag}.
$\eta_\ell>1$ means that K-FAC underestimates the average curvature measured by the empirical Fisher for layer $\ell$, while $\eta_\ell<1$ means that it overestimates it. 
Large variation in $\eta_\ell$ indicates that K-FAC-based singular values are not consistently scaled across layers or heads.
We use this ratio only to expose this mismatch, rather than to rescale individual singular components. 

To empirically examine and quantify this mismatch, we analyze LLaVA-v1.5 7B~\cite{liu2023visual-llava} by computing this ratio for each Q-projection head. For the $h$-th Q head in layer $\ell$, we evaluate $\eta_{\ell,h}$ on the corresponding head-wise weight block.
As shown in Fig.~\ref{fig:eta-diagnostic}, $\eta_{\ell,h}$ varies substantially across both layers and heads. Some heads remain close to $\eta=1$, while others exceed $2$ and approach $3$ in several layers. Heads within the same layer also
show different degrees of mismatch. These results show that K-FAC-based singular values are not on a consistent scale across layers and heads, and therefore may not be directly compared for global rank allocation. This motivates our use of an empirical-Fisher importance score for more accurate component ranking.

\paragraph{Empirical-Fisher singular-value importance.}
Instead of globally ranking singular components by the K-FAC estimate $\sigma_{\ell,i}^2$, we extend the Importance Score~\cite{wang2025qsvd} to the loss-aware SVD setting. Specifically, we use calibration gradients to directly estimate the empirical-Fisher-based loss sensitivity of each singular component, providing a better calibrated and more loss-aware criterion for cross-layer rank allocation.
 Recall that %the curvature-weighted matrix is decomposed as 
$\widetilde{W}=S_xWS_g=U\Sigma V^{\top}$. For the $i$-th singular value of layer $\ell$, we define its importance score as:
{\small
\begin{equation}
\widehat{I}_{\ell,i}
=
\frac{1}{N}
\sum_{n=1}^{N}
\sigma_{\ell,i}^{2}
\left[
U_{\ell}^{\top}
S_{x,\ell}^{-\top}
G_{W_\ell}^{(n)}
S_{g,\ell}^{-\top}
V_{\ell}
\right]_{i,i}^{2}
\end{equation}}
where $G_{W_\ell}^{(n)}=\nabla_{W_\ell}\ell^{(n)}$ denotes the weight gradient of layer $\ell$ on the $n$-th calibration sample. This score measures how much each singular component contributes to the empirical-Fisher quadratic loss approximation, rather than relying on the K-FAC surrogate used to compute the SVD basis. It combines the singular value magnitude, captured by $\sigma_{\ell,i}^{2}$, with its alignment to the calibration gradients after mapping the basis back to the original parameter space. We provide the derivation in Appendix~\ref{sec:appx-importance-score}.

\paragraph{Global rank allocation.}
We allocate ranks by selecting singular components under a target parameter budget $B$. For the $\ell$-th candidate matrix $W_\ell\in\mathbb{R}^{m_\ell\times n_\ell}$, where $m_\ell$ and $n_\ell$ are its input and output dimensions, preserving one singular component will cost additional $\kappa_\ell=m_\ell+n_\ell$ parameters. Let
$\mathcal{S}$ denote the set of selected components, where
$(\ell,i)\in\mathcal{S}$ means that the $i$-th singular component of $W_\ell$ is retained. The global selection problem is:
%(\textcolor{red}{not clear, what is $B$? what is $m_{l}$, $n_{l}$, $\kappa_\ell$? Is $n$ mean $n-th$ calibration sample? what is rank-$1$? }):

\begin{equation}
\small
\mathcal{S}^{\star}
=
\arg\max_{\mathcal{S}}
\sum_{(\ell,i)\in\mathcal{S}}
\widehat{I}_{\ell,i}
\quad
\mathrm{s.t.}
\quad
\sum_{(\ell,i)\in\mathcal{S}}
\kappa_\ell
\le B
\label{eq:global-rank-allocation}
\end{equation}

The rank assigned to candidate layer $\ell$ is
$r_\ell=\left|\{i:(\ell,i)\in\mathcal{S}^{\star}\}\right|$.
If all candidate layer matrices have the same shape, the problem reduces to globally keeping the components with
the largest $\widehat{I}_{\ell,i}$. When candidate matrices have different shapes, retaining a singular component can incur different parameter costs. We therefore rank components by importance per parameter, $\widehat{I}_{\ell,i}/\kappa_\ell$, and greedily retain the highest-ranked components until the parameter budget $B$ is reached.

In LSVD, we treat the per-head QKV projections in self-attention layers as candidate matrices. We apply the loss-aware SVD from Sec.~\ref{sec:loss-aware-svd} to each per-head projection and use the proposed rank allocation strategy to assign ranks across layers and heads. This compresses QKV weights and reduces projection cost. For K/V projections, the per-head low-rank form also lowers KV-cache memory traffic during autoregressive decoding. Overall, LSVD accelerates inference while preserving accuracy by retaining the most important singular components.

\subsection{Quantization-Aware Whitening for SVD}
\label{sec:qa-whitening}
To further reduce computational cost, we apply linear quantization to both activations and low-rank weights. The loss-aware SVD objective in Sec.~\ref{sec:loss-aware-svd} uses the input-side K-FAC factor  $X=\mathbb{E}[x^{\top}x]$, which corresponds to a full-precision linear layer $y=xW$. In deployment, however, activation quantization replaces the input with $\tilde{x}=Q(x)$, where $Q(\cdot)$ is the activation quantization function with its scale and clipping range fixed after calibration. As a result, substituting $W$ with its low-rank approximation $\widehat{W}$ induces the quantized-output perturbation $\Delta y_q=\tilde{x}(\widehat{W}-W)$, rather than $x(\widehat{W}-W)$.
To align the whitening objective with the deployed low-precision layer, we adopt the~\textit{Quantization-Aware Whitening} (QAW) by estimating the input-side K-FAC factors using quantized calibration activations: $X_q=\mathbb{E}[\tilde{x}^{\top}\tilde{x}]$. We then factorize it as $X_q=S_{x,q}^{\top}S_{x,q}$ and construct the transformed weight $\widetilde W_q=S_{x,q}WS_g$. The remaining steps follow Sec.~\ref{sec:loss-aware-svd}: we perform truncated SVD on $\widetilde W_q$ and map the resulting low-rank factors back to the original weight space. After obtaining the low-rank factors, following prior work~\cite{wang2025qsvd}, we insert a rotation matrix~\cite{ashkboos2024quarot} between them before weight quantization, which preserves the low-rank product
while smoothing outliers and improving quantization robustness. When activation quantization is disabled, $Q$ reduces to the identity map, so $X_q=X$, and this formulation recovers the standard loss-aware SVD in Sec.~\ref{sec:loss-aware-svd}.  %(\textcolor{red}{should we mentioned hadamard? how the outlier is smoothed?})

\subsection{Low-rank Approximation for Low-precision Feed-Forward Networks}
\label{sec:mix-quant-svd}

% \begin{wrapfigure}{r}{0.63\linewidth}
% \vspace{-12pt}
% \begin{minipage}{0.98\linewidth}
{%\small
\begin{algorithm}[tb]
\caption{SVD-aware Hybrid FFN Compression}
\label{alg:hybrid-ffn}
\small
\begin{algorithmic}[1]
\Require FFN weights $W_G,W_U \in R^{ d_{\text{in}}\times d_\text{ffn}},W_D \in R^{ d_\text{ffn}\times d_{\text{out}}}$; SVD channel ratio $\gamma$; SVD retention ratio $\rho_{\rm ffn}$
\Ensure Quantized hybrid low-rank--dense FFN

\State $\bar W_G,\bar W_U \leftarrow
\mathrm{LowRank}(W_G;\rho_{\rm ffn}),\mathrm{LowRank}(W_U;\rho_{\rm ffn})$
\Comment{temporary SVD for channel scoring}

\For{$j=1,\ldots,d_{\rm ffn}$}
    \State $e_j \leftarrow
    \|W_G[:,j]-\bar W_G[:,j]\|_2^2
    +\|W_U[:,j]-\bar W_U[:,j]\|_2^2$
    \Comment{SVD error for each hidden channel}
\EndFor

\State $\mathcal{I}_S \leftarrow$ indices of the $\lfloor \gamma d_{\rm ffn}\rfloor$ smallest scores in $\{e_j\}$
\Comment{SVD-friendly channels}
\State $\mathcal{I}_D \leftarrow \{1,\ldots,d_{\rm ffn}\}\setminus\mathcal{I}_S$
\Comment{other channels kept dense}

\State $\pi \leftarrow \mathrm{Concat}(\mathcal{I}_S,\mathcal{I}_D)$
\Comment{new channel order with SVD channels first}
\State $P_\pi \leftarrow$ permutation matrix induced by this order
\State $W_G',W_U',W_D' \leftarrow W_GP_\pi,\; W_UP_\pi,\; P_\pi^\top W_D$
\Comment{preserve FFN output}

\State $[W_{G,S}',W_{G,D}'] \leftarrow \mathrm{Split}(W_G',|\mathcal{I}_S|)$, \quad
$[W_{U,S}',W_{U,D}'] \leftarrow \mathrm{Split}(W_U',|\mathcal{I}_S|)$
\Comment{separate SVD and dense blocks}

\State $(A_G,B_G),(A_U,B_U) \leftarrow
\mathrm{LowRank}(W_{G,S}';\rho_{\rm ffn}),\;
\mathrm{LowRank}(W_{U,S}';\rho_{\rm ffn})$
\Comment{SVD selected blocks}

\For{$T\in\{G,U\}$}
    \State $R_T \leftarrow \mathrm{Rotation}(\mathrm{rank}(A_T))$
    \Comment{orthogonal rotation: $R_T^\top R_T=I$}
    \State $(A_T,B_T) \leftarrow (A_TR_T,R_T^\top B_T)$
    \Comment{preserve product while smoothing factors}
    \State $(A_{T,q},B_{T,q}) \leftarrow \mathrm{RTN}(A_T,B_T)$
\EndFor

\State $W_{G,D,q}',W_{U,D,q}',W_{D,q}' \leftarrow
\mathrm{RTN}(W_{G,D}'),\mathrm{RTN}(W_{U,D}'),\mathrm{RTN}(W_D')$
\Comment{quantize dense branch and down projection}

\State \Return $\{A_{G,q},B_{G,q},A_{U,q},B_{U,q},
W_{G,D,q}',W_{U,D,q}',W_{D,q}'\}$
\end{algorithmic}
\end{algorithm}
% \vspace{-30pt}
}
% \end{minipage}
% \vspace{-12pt}
% \end{wrapfigure}

While Sections~\ref{sec:loss-aware-svd} and~\ref{sec:rank-alloc} focus on optimizing self-attention (SA) layers, this section extends the discussion to gated feed-forward networks (FFNs). FFN layers account for a large fraction of transformer parameters, but applying the same low-rank compression strategy uniformly to all FFN hidden channels can cause substantial accuracy degradation. Empirically, we find that decomposing the entire FFN weight matrices with SVD often introduces excessive degradation from the original model. 
To address this, we apply low-rank SVD only to selected FFN submatrices and introduce a SVD-aware procedure to identify SVD-friendly hidden channels. Afterward, we further quantize the resulting matrices to reduce the computational cost and parameter size of the FFN layers.

Algorithm~\ref{alg:hybrid-ffn} summarizes the proposed hybrid FFN compression procedure, which consists of four steps. First, we estimate the SVD-friendliness of each FFN hidden channel by applying a temporary truncated SVD to the gate and up projections, and select the channels with the lowest reconstruction error for final SVD. Second, we move the selected channels into a contiguous block by applying the same permutation to $W_G$ and $W_U$, and the inverse permutation to $W_D$. This preserves the SwiGLU FFN output because each gate/up hidden channel remains aligned with its corresponding row in the down projection. Third, we apply SVD only to the selected gate/up blocks, while keeping the remaining channels dense. Finally, we insert an orthogonal rotation between each pair of SVD factors and apply round-to-nearest (RTN) quantization to all resulting factors and dense blocks.

\subsection{System Implementation}
\label{sec:mixgaup-kernel}

% \begin{wrapfigure}{r}{0.5\linewidth}
%     \vspace{-12pt}
%     \centering
%     \includegraphics[width=\linewidth]{figures/mixffn_kernel.pdf}
%     \vspace{-15pt}
%     \caption{Mix SVD-Quantization FFN kernel (MixFFN) utilize fused kernel to reduce intermediate activations load/store and kernel launch time. The detailed implmentation includes: (a) naive pytorch implementation of Matmul-SwiGLU operations. (b) fused dense branch $W_{U,D}',W_{G,D}'$ with SwiGLU. (c) fused SVD branch down projection $A_U,A_G$. (d) fused SVD branch up projection $B_U,B_G$ with SwiGLU.}
%     \label{fig:mixffn_kernel}
%     \vspace{-10pt}
% \end{wrapfigure}
\begin{wrapfigure}{r}{0.4\linewidth}
    \vspace{-10pt}
    \centering
    \includegraphics[width=\linewidth]{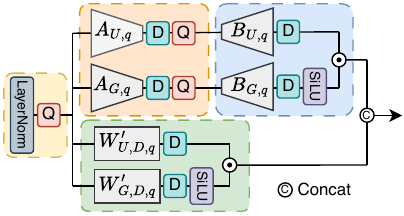}
    \vspace{-15pt}
    \caption{Fused kernels for the hybrid low-rank--dense FFN. $\odot$ denotes the element-wise product in SwiGLU, and Q/D denote quantization/dequantization.}
    \label{fig:mixffn_kernel}
    \vspace{-15pt}
\end{wrapfigure}
To realize the latency benefits of LSVD’s low-rank and low-precision design in Sec.~\ref{sec:mix-quant-svd}, we implement the hybrid low-rank FFN using fused Triton kernels, as shown in Fig.~\ref{fig:mixffn_kernel}. This reduces kernel launches and intermediate memory traffic across gate/up projections, SwiGLU, and quantization/dequantization operations.

%\caption{Mix SVD-Quantization FFN kernel (MixFFN) utilize fused kernel to reduce intermediate activations load/store and kernel launch time. The detailed implmentation includes: (a) naive pytorch implementation of Matmul-SwiGLU operations. (b) fused dense branch $W_{U,D}',W_{G,D}'$ with SwiGLU. (c) fused SVD branch down projection $A_U,A_G$. (d) fused SVD branch up projection $B_U,B_G$ with SwiGLU.}

For the SVD-factorized branch in Sec.~\ref{sec:mix-quant-svd}, 
where $W_{G,S}'\approx A_{G,q}B_{G,q}$ and $W_{U,S}'\approx A_{U,q}B_{U,q}$, we first fuse the input-side projections $XA_{G,q}$ and $XA_{U,q}$ into a single kernel to generate the latent activations.
%$L_G$ and $L_U$. 
The second-stage projections with $B_{G,q}$ and $B_{U,q}$ are fused and then applied to the corresponding latent activations. For the dense branch, we likewise fuse the computation of $XW_{G,D,q}'$ and $XW_{U,D,q}'$ into one kernel.
Both branches use the same channel partition defined in Sec.~\ref{sec:mix-quant-svd}. Since the gate and up projections share aligned channel assignments, SwiGLU can be applied independently within each branch. We therefore fuse SiLU and element-wise multiplication after the projections, and then concatenate the SVD-factorized and dense outputs to obtain the final FFN activation.
For low-precision inference, we further fuse the per-token maximum reduction, quantization, and dequantization steps into the projection and nonlinear LayerNorm kernels, as shown by the Q/D nodes in Fig.~\ref{fig:mixffn_kernel}. This avoids extra passes over intermediate activations and reduces quantization-related overhead. Overall, the fused implementation preserves the original FFN computation while reducing memory movement, kernel launches, and quantization overhead. For SA layers, we extend the fused attention kernel~\cite{wang2026wsvd} to the low-precision setting to further improve inference efficiency.

\section{Evaluation}
\label{sec:eval}
% \subsection{Experimental Setup}
% \label{sec:exp-setup}
We conduct experiments on five representative vision–language models: LLaVA-v1.5 7B~\cite{liu2023visual-llava}, LLaVA-Next 7B, LLaVA-Next 13B, Qwen2-VL 7B~\cite{wang2024qwen2}, and SmolVLM 2B~\cite{marafioti2025smolvlm}. For calibration, we use 256 samples randomly drawn from the ScienceQA training split~\cite{lu2022learn_sqa}, following the procedures described in Sec.~\ref{sec:loss-aware-svd},~\ref{sec:rank-alloc} and~\ref{sec:qa-whitening}. Evaluation is conducted on three widely used benchmarks, ScienceQA-IMG~\cite{lu2022learn_sqa}, SEED-Bench-IMG~\cite{li2024seed} and MMBench~\cite{liu2024mmbench}, consistent with prior studies on VLMs such as LLaVA, using the VLMEvalKit~\cite{duan2024vlmevalkit} tool.
For comparison, LSVD is benchmarked against several baselines, including SVD-based approaches (SVD-LLM~\cite{wang2024svd}, QSVD~\cite{wang2025qsvd}, WSVD~\cite{wang2026wsvd}) and quantization-based techniques (DuQuant~\cite{lin2024duquant}, QVLM~\cite{wang2024qvlm}). For SVD-LLM, QSVD and WSVD, we follow their official implementations and apply SVD independently to the QKV projection matrices to ensure a fair comparison with LSVD, while leaving other linear layers unchanged. In addition, LSVD introduces low-rank approximation for low-precision FFN under quantized settings. For FFN optimization, we set $\gamma=50\%$ and retain $\rho_{\rm ffn}=90\%$ of the selected SVD branch parameters for Algorithm~\ref{alg:hybrid-ffn}.
We denote the parameter ratio by $\rho_1$. For rank-$r$ SVD of a per-head projection $W_h\in\mathbb{R}^{E\times H}$, where $E$ is the hidden dimension and $H$ is the head dimension, $\rho_1=(E+H)r/(EH)$, while KV-cache size is reduced by the rank ratio $r/H$.
More results are shown in the Appendix~\ref{sec:appx-results}.

To isolate the impact of SVD from quantization, we introduce \textbf{LSVD-noQ}, which applies only the loss-aware SVD and rank allocation techniques described in Sec.~\ref{sec:loss-aware-svd} and~\ref{sec:rank-alloc}. We compare it with SVD-LLM, QSVD-noQ and WSVD-noQ (unquantized version of QSVD/WSVD).
We then apply quantization-aware whitening (Sec.~\ref{sec:qa-whitening}), low-rank approximation for low-precision FFN (Sec.~\ref{sec:mix-quant-svd}) and full-model RTN quantization on top of LSVD-noQ, benchmarking against DuQuant, QVLM, QSVD and WSVD.
All experiments are conducted on NVIDIA H100 GPUs. 

\subsection{Results in FP16}
\label{sec:result-fp16}
We first evaluate the FP16 performance of LSVD-noQ under different rank budgets. To ensure fairness, we align the parameter ratio $\rho_1$ across all methods and compare their accuracy performance under the same $\rho_1$. For LSVD, $\rho_1$ is defined in Sec.~\ref{sec:eval}, while for other SVD-based baselines, $\rho_1$ is defined as the proportion of parameters relative to the original model after SVD is applied. 

\begin{table*}[t]
    \newcommand{\bc}{\cellcolor{blue!10}}
    \newcommand{\accvalours}[1]{\fontsize{9.75pt}{10pt}\selectfont \textbf{#1}}
    \centering
    \small
    \caption{Accuracy of different methods under FP16 (detailed results in Appendix~\ref{sec:appx-results-fp16}).}
    %\vspace{-5pt}
    \resizebox{\linewidth}{!}{ %
    \begin{tabular}{l|l|c|c|c|c|c|c|c|c|c|c|c}
    \toprule
    
    \multirow{2}{*}{Acc.} & \multirow{2}{*}{Method} & \multicolumn{5}{c|}{ScienceQA-IMG $\uparrow$} & \multicolumn{5}{c|}{SEED-Bench $\uparrow$} & \multirow{2}{*}{Avg. $\uparrow$ } \\
    \cmidrule{3-12}
    & & $\rho_1: 90\%$ & $\rho_1: 80\%$ & $\rho_1: 70\%$ & $\rho_1: 60\%$ & $\rho_1: 50\%$ 
    & $\rho_1: 90\%$ & $\rho_1: 80\%$ & $\rho_1: 70\%$ & $\rho_1: 60\%$ & $\rho_1: 50\%$ \\
    
    \midrule

    \multirow{5.75}{*}{\rotatebox{90}{\makecell{SmolVLM\\2B}}}

    & SVD-LLM   
        & \accval{40.06\%} & \accval{17.20\%} & \accval{3.82\%} & \accval{0.64\%} & \accval{0.69\%}
        & \accval{32.49\%} & \accval{15.89\%} & \accval{4.60\%} & \accval{3.56\%} & \accval{1.23\%}
        & \accval{12.02\%}\\
    & QSVD-noQ   
        & \accval{77.00\%} & \accval{62.77\%} & \accval{42.59\%} & \accval{9.87\%} & \accval{0.20\%}
        & \accval{64.80\%} & \accval{50.46\%} & \accval{36.24\%} & \accval{3.60\%} & \accval{2.25\%}
        & \accval{34.98\%}\\
    & WSVD-noQ   
        & \accval{76.30\%} & \accval{71.74\%} & \accval{60.93\%} & \accval{39.51\%} & \accval{27.27\%}
        & \accval{65.78\%} & \accval{63.29\%} & \accval{54.45\%} & \accval{29.22\%} & \accval{27.35\%}
        & \accval{51.58\%}\\
    & \bc\textbf{LSVD-noQ} 
        & \bc\accval{\textbf{84.58\%}} & \bc\accval{\textbf{83.69\%}} & \bc\accval{\textbf{82.00\%}} & \bc\accval{\textbf{81.11\%}} & \bc\accval{\textbf{78.14\%}}
        & \bc\accval{\textbf{68.54\%}} & \bc\accval{\textbf{68.10\%}} & \bc\accval{\textbf{67.62\%}} & \bc\accval{\textbf{67.12\%}} & \bc\accval{\textbf{66.17\%}}
        & \bc\accval{\textbf{74.51\%}}\\
    \cline{3-12}
    \rule{0pt}{2.75ex}
    & \textcolor{gray}{FP16} & \multicolumn{5}{c|}{\textcolor{gray}{Accuracy: 84.58\%}} & \multicolumn{5}{c|}{\textcolor{gray}{Accuracy: 68.47\%}} & \textcolor{gray}{\accval{76.53\%}}\\

    % \midrule
    
    % \multirow{5.75}{*}{\rotatebox{90}{\makecell{Qwen2-VL\\7B}}}

    % & SVD-LLM   
    %     & \accval{82.80\%} & \accval{81.11\%} & \accval{80.32\%} & \accval{77.34\%} & \accval{69.31\%}
    %     & \accval{74.74\%} & \accval{74.27\%} & \accval{71.85\%} & \accval{60.83\%} & \accval{44.36\%}
    %     & \accval{71.69\%}\\
    % & QSVD-noQ   
    %     & \accval{84.09\%} & \accval{83.34\%} & \accval{82.40\%} & \accval{80.57\%} & \accval{\textbf{76.38\%}}
    %     & \accval{75.06\%} & \accval{74.97\%} & \accval{72.75\%} & \accval{70.63\%} & \accval{65.30\%}
    %     & \accval{76.55\%}\\
    % & WSVD-noQ   
    %     & \accval{82.99\%} & \accval{81.90\%} & \accval{67.38\%} & \accval{49.68\%} & \accval{15.87\%}
    %     & \accval{75.74\%} & \accval{75.27\%} & \accval{74.53\%} & \accval{72.96\%} & \accval{62.16\%}
    %     & \accval{65.85\%}\\
    % %
    % & \bc\textbf{LSVD-noQ} 
    %     & \bc\accval{\textbf{85.23\%}} & \bc\accval{\textbf{84.28\%}} & \bc\accval{\textbf{83.74\%}} & \bc\accval{\textbf{82.00\%}} & \bc\accval{75.41\%}
    %     & \bc\accval{\textbf{76.22\%}} & \bc\accval{\textbf{76.24\%}} & \bc\accval{\textbf{75.96\%}} & \bc\accval{\textbf{75.14\%}} & \bc\accval{\textbf{69.95\%}}
    %     & \bc\accval{\textbf{78.42\%}}\\
    % \cline{3-12}
    % \rule{0pt}{2.75ex}
    % & \textcolor{gray}{FP16} & \multicolumn{5}{c|}{\textcolor{gray}{Accuracy: 84.38\%}} & \multicolumn{5}{c|}{\textcolor{gray}{Accuracy: 76.29\%}} & \textcolor{gray}{\accval{80.33\%}}\\

    \midrule
    
    \multirow{5.75}{*}{\rotatebox{90}{\makecell{LLaVA-v1.5\\7B}}}   

    & SVD-LLM  
        & \accval{65.44\%} & \accval{63.71\%} & \accval{61.92\%} & \accval{57.41\%} & \accval{55.53\%}
        & \accval{57.89\%} & \accval{57.50\%} & \accval{55.33\%} & \accval{54.64\%} & \accval{55.31\%} 
        & \accval{58.47\%}\\
    & QSVD-noQ   
        & \accval{67.72\%} & \accval{\textbf{68.22\%}} & \accval{67.08\%} & \accval{65.05\%} & \accval{62.37\%} 
        & \accval{59.84\%} & \accval{59.07\%} & \accval{59.78\%} & \accval{59.00\%} & \accval{58.23\%} 
        & \accval{62.64\%}\\
    & WSVD-noQ   
        & \accval{\textbf{68.17\%}} & \accval{67.72\%} & \accval{\textbf{67.28\%}} & \accval{65.89\%} & \accval{65.49\%}
        & \accval{60.10\%} & \accval{\textbf{60.17\%}} & \accval{59.89\%} & \accval{\textbf{60.18\%}} & \accval{60.46\%}
        & \accval{63.54\%}\\
    % & \rowcolor{blue!10}\textbf{WSVD-noQ} 
    & \bc\textbf{LSVD-noQ}
        & \bc\accval{68.12\%} & \bc\accval{68.02\%} & \bc\accval{67.03\%} & \bc\accval{\textbf{67.08\%}} & \bc\accval{\textbf{66.73\%}}
        & \bc\accval{\textbf{60.14\%}} & \bc\accval{60.00\%} & \bc\accval{\textbf{60.55\%}} & \bc\accval{59.82\%} & \bc\accval{\textbf{60.60\%}}
        & \bc\accval{\textbf{63.81\%}}\\
    \cline{3-12}
    \rule{0pt}{2.75ex}
    & \textcolor{gray}{FP16} & \multicolumn{5}{c|}{\textcolor{gray}{Accuracy: 68.01\%}} & \multicolumn{5}{c|}{\textcolor{gray}{Accuracy: 60.18\%}} & \textcolor{gray}{\accval{64.10\%}}  \\
    
    \midrule
    
    \multirow{5.75}{*}{\rotatebox{90}{\makecell{LLaVA-Next\\13B}}} %LLaVA-v1.5-13B

    & SVD-LLM    
        & \accval{72.53\%} & \accval{72.24\%} & \accval{71.74\%} & \accval{71.15\%} & \accval{70.55\%} 
        & \accval{70.76\%} & \accval{70.63\%} & \accval{70.25\%} & \accval{69.96\%} & \accval{69.58\%} 
        & \accval{70.94\%}\\
    & QSVD-noQ   
        & \accval{71.94\%} & \accval{72.14\%} & \accval{71.74\%} & \accval{72.14\%} & \accval{71.79\%} 
        & \accval{71.23\%} & \accval{71.02\%} & \accval{71.06\%} & \accval{70.92\%} & \accval{70.40\%} 
        & \accval{71.44\%}\\
    & WSVD-noQ   
        & \accval{72.88\%} & \accval{72.98\%} & \accval{\textbf{73.57\%}} & \accval{\textbf{73.48\%}} & \accval{73.28\%}
        & \accval{71.29\%} & \accval{71.17\%} & \accval{71.25\%} & \accval{70.95\%} & \accval{70.81\%}
        & \accval{72.17\%}\\
    & \bc\textbf{LSVD-noQ} 
        & \bc\accval{\textbf{73.13\%}} & \bc\accval{\textbf{73.03\%}} & \bc\accval{73.38\%} & \bc\accval{73.13\%} & \bc\accval{\textbf{73.67\%}}
        & \bc\accval{\textbf{71.39\%}} & \bc\accval{\textbf{71.40\%}} & \bc\accval{\textbf{71.33\%}} & \bc\accval{\textbf{71.37\%}} & \bc\accval{\textbf{71.25\%}}
        & \bc\accval{\textbf{72.31\%}}\\
    \cline{3-12}
    \rule{0pt}{2.75ex}
    & \textcolor{gray}{FP16} & \multicolumn{5}{c|}{\textcolor{gray}{Accuracy: 73.23\%}} & \multicolumn{5}{c|}{\textcolor{gray}{Accuracy: 71.30\%}} & \textcolor{gray}{\accval{72.27\%}}\\

    \bottomrule
    \end{tabular}
    }
    \label{tab:svd-result}
    \vspace{-12pt}
\end{table*}
The evaluation results are summarized in Tab.~\ref{tab:svd-result}, with detailed results in Appendix~\ref{sec:appx-results-fp16}. Under the same parameter ratio $\rho_1$, LSVD-noQ surpasses SVD-LLM, QSVD-noQ, and WSVD-noQ in accuracy in most cases. On large-scale models such as LLaVA-Next 13B, LSVD-noQ incurs less than a $1\%$ accuracy drop on ScienceQA-IMG and SEED-Bench compared to the FP16 baseline. Notably, for LLaVA-Next 13B, when $\rho_1 = 50\%$, LSVD-noQ even outperforms the FP16 model on ScienceQA-IMG, exceeding the FP16 baseline by more than $0.4\%$. One possible reason is that low-rank approximation may implicitly reduce hallucinations~\citep{liu2024survey}, although this hypothesis requires further validation.
Furthermore, LSVD-noQ achieves consistently higher average accuracy across datasets and parameter ratios, with its advantage becoming more pronounced as $\rho_1$ decreases. For example, on SmolVLM, LSVD-noQ maintains over $78\%$ accuracy on ScienceQA-IMG, while other baselines fail to produce usable results under the same parameter-ratio settings.

\subsection{Results in Quantized Settings}
\label{sec:result-quant}
\begin{wraptable}{r}{0.6\linewidth}
    \vspace{-15pt}
    \centering
    \newcommand{\bc}{\cellcolor{blue!10}}
    \newcommand{\accvalours}[1]{\fontsize{9.75pt}{10pt}\selectfont \textbf{#1}}
    \small
    \caption{Accuracy evaluation of different methods under low-precision on LLaVA-v1.5 7B, Next 7B and 13B.}%(\textcolor{red}{remove the results on v1.5 13B, using wraptable})}
    \resizebox{0.98\linewidth}{!}
    {%
    \begin{tabular}{l|c|c|c|c|c|c|c}
    \toprule
    
    \multirow{2}{*}{Method} 
    & \multicolumn{3}{c|}{ScienceQA-IMG $\uparrow$} 
    & \multicolumn{3}{c|}{SEED-Bench $\uparrow$} 
    & \multirow{2}{*}{Avg. $\uparrow$} \\
    
    \cmidrule{2-7}
    
    & v1.5 7B & Next 7B & Next 13B 
    & v1.5 7B & Next 7B & Next 13B 
    & \\
    
    \midrule
    
    DuQuant   
        & \accval{57.36\%} 
        & \accval{66.34\%} 
        & \accval{70.20\%} 
        & \accval{54.11\%} 
        & \accval{63.64\%} 
        & \accval{66.15\%} 
        & \accval{62.97\%} \\
        
    QVLM 
        & \accval{55.24\%} 
        & \accval{60.60\%} 
        & \accval{65.28\%} 
        & \accval{50.13\%} 
        & \accval{50.38\%} 
        & \accval{65.39\%} 
        & \accval{57.84\%} \\
        
    QSVD 
        & \accval{65.61\%} 
        & \accval{66.10\%} 
        & \accval{70.43\%} 
        & \accval{58.49\%} 
        & \accval{65.63\%} 
        & \accval{69.21\%} 
        & \accval{65.91\%} \\
        
    WSVD 
        & \accval{64.25\%} 
        & \accval{66.94\%} 
        & \accval{\textbf{73.08\%}} 
        & \accval{60.23\%} 
        & \accval{67.49\%} 
        & \accval{70.67\%} 
        & \accval{67.11\%} \\
        
    \rowcolor{blue!10}
    \textbf{LSVD} 
        & \accval{\textbf{66.14\%}} 
        & \accval{\textbf{68.37\%}} 
        & \accval{72.11\%} 
        & \accval{\textbf{61.94\%}} 
        & \accval{\textbf{67.49\%}} 
        & \accval{\textbf{71.56\%}} 
        & \accval{\textbf{67.94\%}} \\
        
    \cline{2-8}
    
    \rule{0pt}{2.5ex}
    \textcolor{gray}{FP16} 
        & \textcolor{gray}{\accval{68.10\%}} 
        & \textcolor{gray}{\accval{69.60\%}} 
        & \textcolor{gray}{\accval{73.23\%}} 
        & \textcolor{gray}{\accval{60.18\%}} 
        & \textcolor{gray}{\accval{69.02\%}} 
        & \textcolor{gray}{\accval{71.30\%}} 
        & \textcolor{gray}{\accval{68.57\%}} \\
    
    \bottomrule
    \end{tabular}
    }
    \label{tab:svd-qat-result}
    \vspace{-10pt}
\end{wraptable}
% We present results under two weight–activation quantization configurations: W8A8 for LSVD and WSVD with rank settings $\rho_{1}=50\%$, where SVD further reduces the
% full-dimensional A8 KV cache by about $50\%$, and W8A4 for all other baselines. 
We present results under two weight--activation quantization configurations.
LSVD and WSVD use W8A8 with $\rho_{1}=50\%$: A8 halves the KV-cache storage, and SVD keeps nearly half of the K/V channels, yielding about $25\%$ of the FP16 full KV-cache size. Other baselines use W8A4 with a full-dimensional
KV cache.
This design keeps cache size and parameter size comparable across methods, 
while LSVD’s FFN SVD (Sec.~\ref{sec:mix-quant-svd}) further reduces its parameter budget, ensuring fairness in comparison. 

%For QASVD and QSVD, we adopt $\rho_1 = 100\%$ truncation to match the W8 parameter budget of other baselines. Since these methods do not employ per-head SVD, they cannot leverage latent cache representations to accelerate inference and must store the full KV cache. As a result, they are also evaluated under the A4 setting to ensure comparable cache sizes. sÍ

For activation quantization, we adopt per-token symmetric quantization. For weight quantization, we employ RTN with per-channel symmetric scaling and a learnable clipping ratio, where the clipping value is selected via linear search to minimize squared error, following QuaRot~\citep{ashkboos2024quarot}. This quantization scheme is applied to the per-head Q/K/V weight matrices and all remaining attention and feed-forward modules, ensuring that the dominant matrix multiplications in each transformer block are executed in low precision. As shown in Tab.~\ref{tab:svd-qat-result}, LSVD consistently outperforms the baselines in most cases, despite using a smaller parameter budget and the same cache size. On average across models and datasets, LSVD incurs only a modest accuracy drop of just about $0.5\%$ relative to the FP16 baseline, while reducing cache size to $25\%$ of the FP16 model. %At the same time, LSVD achieves more than 0.8\% higher average accuracy than all baselines, demonstrating that the integration of loss-aware per-head SVD, hybrid low-rank FFN, and quantization within LSVD minimizes accuracy degradation.

\subsection{Ablation Studies}
\label{sec:ablation}
\paragraph{Effectiveness of Loss-aware SVD}

\begin{wraptable}{r}{0.54\textwidth}
\vspace{-15pt}
\centering
\small
\caption{Results of loss-aware SVD ablation.}
%\vspace{-5pt}
\resizebox{\linewidth}{!}{ %
\begin{tabular}{c|c|c|c|c|c|c}
\toprule
\multirow{2}{*}{Method} & \multicolumn{5}{c|}{ScienceQA-IMG $\uparrow$} & \multirow{2}{*}{Avg. $\uparrow$} \\
\cmidrule{2-6}
& v1.5 7B & Next 7B & Next 13B & Q2 7B & S 2B \\
\midrule
\textcolor{gray}{FP16}
& \textcolor{gray}{\accval{68.10\%}} 
& \textcolor{gray}{\accval{69.60\%}} 
& \textcolor{gray}{\accval{73.23\%}} 
& \textcolor{gray}{\accval{84.38\%}} 
& \textcolor{gray}{\accval{84.58\%}} 
& \textcolor{gray}{\accval{75.98\%}} \\

Vanilla
& \accval{62.87\%} 
& \accval{67.72\%} 
& \accval{71.28\%} 
& \accval{53.25\%} 
& \accval{0.00\%} 
& \accval{51.02\%} \\

LSVD
& \accval{\textbf{66.73\%}} 
& \accval{\textbf{70.80\%}} 
& \accval{\textbf{73.67\%}} 
& \accval{\textbf{75.41\%}} 
& \accval{\textbf{78.14\%}} 
& \accval{\textbf{72.95\%}} \\
\bottomrule
\end{tabular}
}
\label{tab:loss-aware-svd-ablation}
\vspace{-10pt}
\end{wraptable}

We evaluate loss-aware SVD against a vanilla SVD baseline using the same Importance Score-based rank allocation and compression ratio $\rho_1=50\%$ under FP16. The \textit{Vanilla} baseline applies SVD directly to the weights, while LSVD performs SVD in the loss-aware transformed space from Sec.~\ref{sec:loss-aware-svd}. As shown in Tab.~\ref{tab:loss-aware-svd-ablation}, LSVD consistently outperforms vanilla SVD across all models, improving average accuracy from $51.02\%$ to $72.95\%$, an $21.93\%$ gain. The improvement is especially large on Qwen2-VL 7B and SmolVLM 2B, where vanilla SVD degrades severely, while LSVD preserves much higher accuracy.

%We evaluate the effectiveness of loss-aware SVD by comparing it with a vanilla SVD baseline under the same Importance Score-based rank allocation strategy and compression ratio $\rho_1=50\%$. The \textit{Vanilla} baseline applies standard SVD to the weight matrices, while LSVD performs SVD in the loss-aware transformed space described in Sec.~\ref{sec:loss-aware-svd}. As shown in Tab.~\ref{tab:loss-aware-svd-ablation}, LSVD consistently outperforms vanilla SVD across all evaluated models. On average, LSVD improves the accuracy from 54.17\% to 72.73\%, yielding an 18.56\% gain. The improvement is especially significant on Qwen2-VL 7B and SmolVLM 2B, where vanilla SVD suffers severe degradation under aggressive compression, while LSVD preserves much higher accuracy.

% This demonstrates that loss-aware SVD provides a substantially better low-rank approximation under the same setting.  Moreover, LSVD remains only 2.56\% below the FP16 baseline on average, demonstrating that loss-aware whitening provides a substantially better low-rank approximation than standard SVD under the same compression setting.

\paragraph{Effectiveness of Loss-aware Cross-layer Rank Allocation}

\begin{wraptable}{r}{0.4\textwidth}
\vspace{-15pt}
\centering
\small
\caption{Ablation of rank allocation.}
%\vspace{-5pt}
\resizebox{\linewidth}{!}{ %
\begin{tabular}{c|c|c|c|c}
\toprule
\multirow{2}{*}{Method} & \multicolumn{3}{c|}{ScienceQA-IMG $\uparrow$} & \multirow{2}{*}{Avg. $\uparrow$} \\
\cmidrule{2-4}
& v1.5 7B & Next 7B & Next 13B & \\
\midrule
\textcolor{gray}{FP16}
& \textcolor{gray}{\accval{68.10\%}} 
& \textcolor{gray}{\accval{69.60\%}} 
& \textcolor{gray}{\accval{73.23\%}} 
& \textcolor{gray}{\accval{70.31\%}} \\

Uniform
& \accval{58.06\%} 
& \accval{59.84\%} 
& \accval{60.63\%} 
& \accval{59.51\%} \\

SV-based
& \accval{59.05\%} 
& \accval{55.63\%} 
& \accval{68.22\%} 
& \accval{60.97\%} \\

LSVD
& \accval{\textbf{66.14\%}} 
& \accval{\textbf{68.34\%}} 
& \accval{\textbf{72.11\%}} 
& \accval{\textbf{68.86\%}} \\
\bottomrule
\end{tabular}
}
\label{tab:rank-allocation-ablation}
\vspace{-10pt}
\end{wraptable}
We evaluate the loss-aware cross-layer rank allocation in Sec.~\ref{sec:rank-alloc} by comparing it against two allocation baselines under W8A8. The \textit{Uniform} baseline keeps the same number of top-$k$ singular components in each layer, without performing cross-layer rank allocation. The \textit{SV-based} baseline follows the naive singular-value-based strategy discussed in the beginning of Sec.~\ref{sec:rank-alloc}, which ranks components globally only by their singular values. As shown in Tab.~\ref{tab:rank-allocation-ablation}, LSVD consistently outperforms both baselines across all evaluated models. On average, LSVD improves over uniform allocation by $9\%$ and over SV-based allocation by nearly $8\%$. %The gains are especially clear on LLaVA-Next 7B and LLaVA-Next 13B, where the naive allocation baselines suffer from substantial accuracy degradation. Moreover, LSVD achieves an average accuracy of 69.51\%, only 1.18\% below the FP16 baseline, demonstrating that loss-aware rank allocation effectively preserves model accuracy under the same compression setting.
Additional ablations in Appendix~\ref{sec:appx-results-ablation} further show
that QAW improves low-precision SVD, while SVD-aware permutation preserves accuracy under hybrid FFN compression.

\subsection{System Evaluation}
\label{sec:system-evaluation}
\begin{wrapfigure}{r}{0.7\linewidth}
    \vspace{-10pt}
    \centering
    \includegraphics[width=\linewidth]{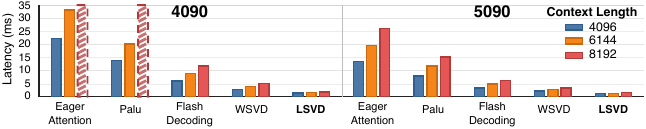}
    \vspace{-15pt}
    \caption{Latency evaluations on RTX 4090 and 5090. Dashed-line columns indicate out-of-memory cases.}
    \label{fig:eval-latency}
    %\vspace{-10pt}
\end{wrapfigure} 
We evaluate the system-level performance of LSVD under the quantized setting in Sec.~\ref{sec:result-quant}, focusing on decoding acceleration. We measure layer-wise decoding latency of LLaVA-Next 7B across attention and FFN modules using our fused kernels from Sec.~\ref{sec:mixgaup-kernel} on RTX 4090 and 5090 GPUs. We compare against Eager Attention without Flash Decoding, Palu~\citep{chang2024palu}, Flash Decoding~\citep{dao2023flashdecoding}, and WSVD~\cite{wang2026wsvd}. Flash Decoding uses SDPA, while Palu and WSVD use their official 
implementations. Eager Attention and Flash Decoding operate on the full KV cache. For low-rank methods, Palu and LSVD use $\rho_1=50\%$, reducing KV cache size by about $50\%$, 
while WSVD uses $\rho_1=60\%$ to match LSVD more closely in accuracy. Both WSVD and LSVD use INT8 W8A8 quantization, while others operate at FP16, and all measurements use batch size $32$.
As shown in Fig.~\ref{fig:eval-latency}, LSVD consistently outperforms all baselines on both GPUs, achieving on average a $4.7\times$ speedup over Flash Decoding and a $2.3\times$ speedup over WSVD. The gains come from the fused kernels in Sec.~\ref{sec:mixgaup-kernel}: the W8A8 low-rank decoding kernel reduces attention-side materialization and KV-cache memory traffic, while the hybrid FFN kernel fuses low-rank reconstruction, SwiGLU, quantization, and accumulation. By reducing memory movement and kernel launch overhead, LSVD delivers consistent decoding latency reduction across context lengths and GPUs while preserving accuracy.

\section{Conclusion and Limitation}
\label{sec:conclusion}
LSVD provides a loss-aware low-rank compression framework for efficient VLM inference by aligning SVD with model-loss sensitivity, allocating ranks across layers and heads using calibration gradients, and extending compression to FFN layers with a hybrid SVD--quantization design. Currently, LSVD is evaluated on representative VLMs and benchmarks, and broader validation on newer architectures, longer contexts, and more aggressive low-bit settings remains future work.

{
\newpage
\small
\bibliographystyle{plainnat}
\bibliography{references}
}

%%%%%%%%%%%%%%%%%%%%%%%%%%%%%%%%%%%%%%%%%%%%%%%%%%%%%%%%%%%%
\newpage
\appendix

\section{Technical appendices and supplementary material}
% Technical appendices with additional results, figures, graphs, and proofs may be submitted with the paper submission before the full submission deadline (see above). You can upload a ZIP file for videos or code, but do not upload a separate PDF file for the appendix. There is no page limit for the technical appendices. 

% Note: Think of the appendix as ``optional reading'' for reviewers. The paper must be able to stand alone without the appendix; for example, adding critical experiments that support the main claims to an appendix is inappropriate. 
%\section{Appendix}

\subsection{K-FAC Loss Surrogate for Weight Compression}
\label{sec:appx-loss-surrogate}

\paragraph{Derivation.}
By Taylor's theorem, the loss variation induced by $\Delta W$ satisfies:
{\small
\begin{equation}
\Delta \mathcal{L}
=
\mathcal{L}(W+\Delta W)-\mathcal{L}(W)  \\
=
\langle \nabla_W\mathcal{L}(W),\Delta W\rangle
+
\frac{1}{2}
\mathrm{vec}(\Delta W)^{\top}
H_W
\mathrm{vec}(\Delta W)
+
O(\|\Delta W\|_F^3)
\label{eq:taylor-expansion}
\end{equation}
}
where $H_W=\nabla_W^2\mathcal{L}(W)$. For a pretrained model close to a
stationary point and a small compression perturbation, the first-order term and
higher-order remainder are negligible, giving:
{\small
\begin{equation}
\Delta \mathcal{L}
\approx
\frac{1}{2}
\mathrm{vec}(\Delta W)^{\top}
H_W
\mathrm{vec}(\Delta W)
\label{eq:hessian-loss-surrogate}
\end{equation}
}

Since directly forming the Hessian block $H_W$ is prohibitively expensive for large models, following standard natural-gradient and second-order approximations, we replace
the Hessian block with the Fisher block~\cite{amari1998natural,martens2020new},
and then use K-FAC to obtain a tractable Kronecker-factored curvature
surrogate~\cite{martens2015optimizing}. For the linear layer $y=xW$, this gives:
{\small
\begin{equation}
\begin{aligned}
H_W
&\approx F_W
=
\mathbb{E}
\!\left[
\mathrm{vec}(\nabla_W\ell)
\mathrm{vec}(\nabla_W\ell)^{\top}
\right], \qquad
\nabla_W\ell
=x^{\top}g,
\qquad
\mathrm{vec}(\nabla_W\ell)=g^{\top}\otimes x^{\top}, \\
F_W
&=
\mathbb{E}
\!\left[
(g^{\top}\otimes x^{\top})(g\otimes x)
\right] =
\mathbb{E}
\!\left[
(g^{\top}g)\otimes(x^{\top}x)
\right]
\approx
\mathbb{E}[g^{\top}g]\otimes \mathbb{E}[x^{\top}x]
=
G\otimes X
\end{aligned}
\label{eq:fisher-linear-compact}
\end{equation}
}
Here $X=\mathbb{E}[x^{\top}x]$ and $G=\mathbb{E}[g^{\top}g]$ are estimated from calibration data. Substituting $H_W\approx G\otimes X$ into Eq.~\ref{eq:hessian-loss-surrogate} yields Eq.~\ref{eq:kfac-loss-surrogate}.

\subsection{Proof of Theorem~\ref{thm:loss-aware-svd}}
\begin{proof}
Let $\Delta W=W-\widehat W_r$. From
Assumption~\ref{assm:kfac-loss-surrogate},
using $X=S_x^{\top}S_x$ and $G=S_gS_g^{\top}$, together with the Kronecker-product identities $(A\otimes B)(C\otimes D)=(AC)\otimes(BD)$ and $\mathrm{vec}(ABC)=(C^{\top}\otimes A)\mathrm{vec}(B)$, we obtain:
{\small
\begin{equation}
\begin{aligned}
\Delta\mathcal{L}
&\approx
\frac{1}{2}\mathrm{vec}(\Delta W)^{\top}
(G\otimes X)
\mathrm{vec}(\Delta W) 
=
\frac{1}{2}\mathrm{vec}(\Delta W)^{\top}
(S_gS_g^{\top}\otimes S_x^{\top}S_x)
\mathrm{vec}(\Delta W) \\
&=
\frac{1}{2}\mathrm{vec}(\Delta W)^{\top}
(S_g\otimes S_x^{\top})(S_g^{\top}\otimes S_x)
\mathrm{vec}(\Delta W) 
=
\frac{1}{2}\left\|
(S_g^{\top}\otimes S_x)\mathrm{vec}(\Delta W)
\right\|_2^2 \\
&=
\frac{1}{2}\left\|
\mathrm{vec}(S_x\Delta W S_g)
\right\|_2^2
=
\frac{1}{2}\left\|
S_x\Delta W S_g
\right\|_F^2
\label{eq:kfac-to-weighted-frobenius}
\end{aligned}
\end{equation}
}
Therefore, minimizing the estimated loss increase in Assumption~\ref{assm:kfac-loss-surrogate} over rank-$r$ approximations is equivalent to:
{\small
\begin{equation}
\min_{\operatorname{rank}(\widehat W_r)\le r}
\left\|S_x(W-\widehat W_r)S_g\right\|_F^2
\label{eq:loss-aware-svd-objective}
\end{equation}
}
Define the curvature-weighted weight and rank-$r$ approximation as $\widetilde W=S_xWS_g$ and $\widetilde W_r=S_x\widehat W_rS_g$. Since $S_x$ and $S_g$ are nonsingular, left and right multiplication by them is 
% one-to-one and 
rank-preserving. Thus Eq.~\ref{eq:loss-aware-svd-objective} reduces to the standard Frobenius-norm rank-$r$ approximation problem:
{\small
\begin{equation}
\min_{\operatorname{rank}(\widetilde W_r)\le r}
\left\|\widetilde W-\widetilde W_r\right\|_F^2
\label{eq:standard-rank-r-proof}
\end{equation}
}
By the Eckart--Young--Mirsky theorem~\cite{eckart1936approximation,mirsky1960symmetric},
the optimal solution is the truncated SVD
$\widetilde W_r^{\star}=U_r\Sigma_rV_r^{\top}$. Mapping it back to the original weight space gives:
{\small
\begin{equation}
\widehat W_r^{\star}
=
S_x^{-1}\widetilde W_r^{\star}S_g^{-1}
=
S_x^{-1}U_r\Sigma_rV_r^{\top}S_g^{-1}
\end{equation}
}
Moreover, the optimal residual in the transformed space is
$\|\widetilde W-\widetilde W_r^{\star}\|_F^2=\sum_{i=r+1}^{q}\sigma_i^2$.
Substituting this residual into Eq.~\ref{eq:kfac-to-weighted-frobenius} gives:
$\Delta\mathcal{L}_r\approx\frac{1}{2}\sum_{i=r+1}^{q}\sigma_i^2$.
This proves Theorem~\ref{thm:loss-aware-svd}.
\end{proof}

\subsection{K-FAC Approximation Error Analysis}
\label{sec:appx-scale-diag}

We analyze why the local loss estimate
$\frac{1}{2}\sigma_{\ell,i}^{2}$ obtained in the loss-aware transformed space
is useful within one layer or head, but may not be directly comparable across
different layers or heads. Let $F_\ell$ denote the empirical Fisher matrix of
layer $\ell$, and let
\[
K_\ell = G_\ell \otimes X_\ell
\]
be the Kronecker-factored approximation used in
Assumption~\ref{assm:kfac-loss-surrogate}. We write the difference between the
empirical Fisher and this K-FAC approximation as:
\[
F_\ell = K_\ell + R_\ell,
\qquad
R_\ell \triangleq F_\ell - K_\ell
\]
Here, $R_\ell$ contains the curvature information that is not captured by the
Kronecker factorization.

For the $i$-th loss-aware singular value of layer $\ell$, $\sigma_{\ell,i}$, let:
\[
W_{\ell,i}^{\mathrm{svd}}
=
S_{x,\ell}^{-1}
\sigma_{\ell,i}u_{\ell,i}v_{\ell,i}^{\top}
S_{g,\ell}^{-1}
\]
denote the corresponding rank-one component mapped back to the original
weight space. Dropping this component induces the weight perturbation:
\[
\Delta W_{\ell,i}^{\mathrm{drop}}
=
-W_{\ell,i}^{\mathrm{svd}},
\qquad
z_{\ell,i}
=
\operatorname{vec}(\Delta W_{\ell,i}^{\mathrm{drop}}).
\]
Here, $W_{\ell,i}^{\mathrm{svd}}$ denotes the component itself, while
$\Delta W_{\ell,i}^{\mathrm{drop}}$ denotes the perturbation caused by
removing it. The minus sign does not affect the following quadratic loss
estimate.

For this discarded loss-aware singular component 
$z_{\ell,i}=\operatorname{vec}(\Delta W_{\ell,i}^{\mathrm{drop}})$, the
empirical-Fisher second-order loss estimate can be decomposed as:
{\small
\begin{equation}
\begin{aligned}
\Delta\mathcal{L}^{F}_{\ell,i}
&\approx
\frac{1}{2} z_{\ell,i}^{\top}F_\ell z_{\ell,i}  \\
&=
\frac{1}{2} z_{\ell,i}^{\top}K_\ell z_{\ell,i}
+
\frac{1}{2} z_{\ell,i}^{\top}(F_\ell-K_\ell)z_{\ell,i} \\
&\triangleq
T^{\mathrm{KFAC}}_{\ell,i}
+
T^{\mathrm{res}}_{\ell,i}
\end{aligned}
\label{eq:fisher-kfac-decomp}
\end{equation}
}
By Theorem~\ref{thm:loss-aware-svd}, the K-FAC-based term is:
\[
T^{\mathrm{KFAC}}_{\ell,i}
=
\frac{1}{2}
\|S_{x,\ell}\Delta W_{\ell,i}^{\mathrm{drop}}S_{g,\ell}\|_F^2
=
\frac{1}{2}\sigma_{\ell,i}^{2}
\]
Thus, $\frac{1}{2}\sigma_{\ell,i}^{2}$ is the loss increase predicted by the
K-FAC surrogate when the $i$-th loss-aware singular component is removed. This
estimate is suitable for comparing components within the same layer or head,
because all candidates are evaluated under the same local approximation.
However, when components from different layers or heads are ranked together,
the K-FAC surrogate in each layer or head may over- or under-estimate the
empirical-Fisher loss by a different amount.

To make this difference explicit, 
we compare the empirical-Fisher and K-FAC approximation using the trace ratio:
%we compare the empirical-Fisher
%curvature with the K-FAC curvature using the trace ratio:
{\small
\begin{equation}
\eta_\ell
=
\frac{\operatorname{Tr}(F_\ell)}
{\operatorname{Tr}(K_\ell)}
=
\frac{\operatorname{Tr}(F_\ell)}
{\operatorname{Tr}(G_\ell\otimes X_\ell)}
=
\frac{\operatorname{Tr}(F_\ell)}
{\operatorname{Tr}(G_\ell)\operatorname{Tr}(X_\ell)}
\label{eq:eta-diagnostic}
\end{equation}
}
The trace measures the total curvature mass of a positive semi-definite matrix.
Therefore, $\eta_\ell$ indicates whether the K-FAC surrogate tends to under- or
over-estimate the empirical-Fisher loss scale on average. Specifically,
$\eta_\ell>1$ means that K-FAC underestimates the total empirical-Fisher
curvature of layer $\ell$, while $\eta_\ell<1$ means that it overestimates it.

\paragraph{Derivation.}
To interpret this trace ratio, consider an isotropic perturbation $z=au$ in the
vectorized weight space of layer $\ell$, where $a>0$ and
$\mathbb{E}_{u}[uu^\top]=I/d_\ell$. For any symmetric curvature matrix $C$, the
quadratic loss surrogate along $z=au$ satisfies:
{\small
\begin{equation}
\mathbb{E}_{u}\!\left[\frac{1}{2}z^\top C z\right]
=
\frac{a^2}{2}\mathbb{E}_{u}[u^\top C u]
\label{eq:eta-indicator-proof-step1}
\end{equation}
}
Using $u^\top C u=\operatorname{Tr}(Cuu^\top)$ and
$\mathbb{E}_{u}[uu^\top]=I/d_\ell$, we have:
{\small
\begin{equation}
\mathbb{E}_{u}[u^\top C u]
=
\operatorname{Tr}\!\left(C\mathbb{E}_{u}[uu^\top]\right)
=
\frac{1}{d_\ell}\operatorname{Tr}(C)
\label{eq:eta-indicator-proof-step2}
\end{equation}
}
Therefore,
\[
\mathbb{E}_{u}\!\left[\frac{1}{2}z^\top C z\right]
=
\frac{a^2}{2d_\ell}\operatorname{Tr}(C)
\]
Applying this result to $C=F_\ell$ and $C=K_\ell$ cancels the common factor
$a^2/(2d_\ell)$ and gives:
{\small
\begin{equation}
\frac{
\mathbb{E}_{u}\!\left[\frac{1}{2}z^\top F_\ell z\right]
}{
\mathbb{E}_{u}\!\left[\frac{1}{2}z^\top K_\ell z\right]
}
=
\frac{\operatorname{Tr}(F_\ell)}
{\operatorname{Tr}(K_\ell)}
=
\eta_\ell
\label{eq:eta-scale-ratio}
\end{equation}
}
Similarly, applying the same result to $C=F_\ell-K_\ell$ gives the
direction-averaged relative error of the K-FAC curvature:
{\small
\begin{equation}
\frac{
\mathbb{E}_{u}\!\left[\frac{1}{2}z^\top (F_\ell-K_\ell) z\right]
}{
\mathbb{E}_{u}\!\left[\frac{1}{2}z^\top K_\ell z\right]
}
=
\eta_\ell - 1
\label{eq:eta-residual-ratio}
\end{equation}
}
Finally, since $K_\ell=G_\ell\otimes X_\ell$, the Kronecker trace identity gives:
\[
\operatorname{Tr}(K_\ell)
=
\operatorname{Tr}(G_\ell\otimes X_\ell)
=
\operatorname{Tr}(G_\ell)\operatorname{Tr}(X_\ell)
\]
which yields Eq.~\ref{eq:eta-diagnostic}.

This derivation shows that $\eta_\ell$ is a simple indicator of how the
K-FAC-based loss estimate compares with the empirical-Fisher loss estimate on
average. It is not a per-component correction, because the exact residual term
$T^{\mathrm{res}}_{\ell,i}$ depends on the direction of the discarded
component. However, large variation in $\eta_\ell$ across layers or heads means
that $\frac{1}{2}\sigma_{\ell,i}^{2}$ may not put all components on the same
loss scale. Therefore, directly sorting components from different layers or
heads by $\sigma_{\ell,i}^{2}$ can lead to inaccurate global rank allocation.

The trace-ratio indicator is also easy to compute in practice. Since the trace
only depends on diagonal entries, we do not need to materialize the full
empirical Fisher. Instead, we only need its diagonal entries, which can be
estimated from calibration gradients. For each calibration sample $n$, we run
a forward and backward pass, collect the per-sample gradient
$\nabla_{W_\ell}\ell_n$, square it elementwise, and accumulate it over the
calibration set:
\[
\mathbf f_{\ell}^{\mathrm{diag}}
\approx
\frac{1}{N}
\sum_{n=1}^{N}
\operatorname{vec}(\nabla_{W_\ell}\ell_n)
\odot
\operatorname{vec}(\nabla_{W_\ell}\ell_n)
\]
where $\mathbf f_{\ell}^{\mathrm{diag}}$ is a vector that estimates the
diagonal entries of the empirical Fisher, and $\odot$ denotes elementwise
multiplication. Summing its entries gives the empirical-Fisher trace:
{\small
\begin{equation}
\operatorname{Tr}(F_\ell)
=
\sum_j [F_\ell]_{j,j}
\approx
\sum_j [\mathbf f_{\ell}^{\mathrm{diag}}]_j
=
\frac{1}{N}
\sum_{n=1}^{N}
\left\|
\nabla_{W_\ell}\ell_n
\right\|_F^2 .
\label{eq:diag-fisher-trace}
\end{equation}
}
Thus, the practical computation only requires two trace values:
$\operatorname{Tr}(F_\ell)$ and $\operatorname{Tr}(K_\ell)$. The
empirical-Fisher trace $\operatorname{Tr}(F_\ell)$ is estimated by accumulating
squared per-sample calibration gradients, while the K-FAC trace
$\operatorname{Tr}(K_\ell)$ is computed directly from the Kronecker factors:
\[
\operatorname{Tr}(K_\ell)
=
\operatorname{Tr}(G_\ell\otimes X_\ell)
=
\operatorname{Tr}(G_\ell)\operatorname{Tr}(X_\ell)
\]
We estimate both quantities on the same calibration set used for loss-aware SVD
and rank allocation. This analysis supports the observation in
Sec.~\ref{sec:rank-alloc}: the local loss estimate
$\frac{1}{2}\sigma_{\ell,i}^{2}$ is useful for deciding which components to
truncate within a layer or head, but it is not always reliable for ranking
components globally. This motivates the empirical-Fisher importance score
derived in Appendix~\ref{sec:appx-importance-score}, which keeps the
loss-aware SVD basis while calibrating each singular component on a shared
empirical loss scale.

\subsection{Loss-Aware Rank Allocation}
\label{sec:appx-importance-score}

We provide the derivation of the calibration-gradient-based importance score
used for global rank allocation. For a target layer $\ell$, let the
loss-aware SVD of the curvature-weighted weight be
$\widetilde W_\ell=S_{x,\ell}W_\ell S_{g,\ell}
=U_\ell \Sigma_\ell V_\ell^\top$, where
$\Sigma_\ell=\operatorname{diag}(\sigma_{\ell,1},\ldots,\sigma_{\ell,q_\ell})$.
The rank-one component corresponding to the $i$-th singular value $\sigma_{\ell,i}$ in the original weight space is:
\[
W_{\ell,i}^{\mathrm{svd}}
=
S_{x,\ell}^{-1}
\sigma_{\ell,i}u_{\ell,i}v_{\ell,i}^{\top}
S_{g,\ell}^{-1}
\]
Under the local K-FAC surrogate, discarding this component gives the
loss estimate $\frac{1}{2}\sigma_{\ell,i}^{2}$. As discussed in
Sec.~\ref{sec:rank-alloc}, this quantity is meaningful for
local truncation within the same layer or head, but its scale may not be
directly comparable across different layers or heads.

To obtain a shared empirical scale for global rank allocation, we estimate the
loss increase of removing each loss-aware singular component under the empirical
Fisher. Recall that LSVD performs SVD on the curvature-whitened weight
$\widetilde W_\ell=S_{x,\ell}W_\ell S_{g,\ell}
=U_\ell\Sigma_\ell V_\ell^\top$. We introduce a binary mask $m_{\ell,i}$ for
each singular component in this whitened space:
\[
\widetilde W_\ell(m)
=
\sum_{i=1}^{q_\ell}
m_{\ell,i}\sigma_{\ell,i}u_{\ell,i}v_{\ell,i}^{\top}
\]
Mapping it back to the original weight space gives:
\[
\widehat W_\ell(m)
=
S_{x,\ell}^{-1}\widetilde W_\ell(m)S_{g,\ell}^{-1}
=
\sum_{i=1}^{q_\ell}
m_{\ell,i}
S_{x,\ell}^{-1}
\sigma_{\ell,i}u_{\ell,i}v_{\ell,i}^{\top}
S_{g,\ell}^{-1}
\]
Therefore, removing the $i$-th loss-aware singular component corresponds to the
weight perturbation:
\[
\Delta W_{\ell,i}^{\mathrm{drop}}
=
-W_{\ell,i}^{\mathrm{svd}}
=
-
S_{x,\ell}^{-1}
\sigma_{\ell,i}u_{\ell,i}v_{\ell,i}^{\top}
S_{g,\ell}^{-1}
\]

Using the second-order loss approximation with the empirical Fisher
$F_\ell=\mathbb{E}[\mathrm{vec}(\nabla_{W_\ell}\ell)
\mathrm{vec}(\nabla_{W_\ell}\ell)^\top]$, the loss increase caused by this perturbation is:
\[
\Delta \mathcal{L}_{\ell,i}
\approx
\frac{1}{2}
\mathrm{vec}(\Delta W_{\ell,i}^{\mathrm{drop}})^\top
F_\ell
\mathrm{vec}(\Delta W_{\ell,i}^{\mathrm{drop}})
\]
Substituting the empirical Fisher form gives:
\[
\Delta \mathcal{L}_{\ell,i}
\approx
\frac{1}{2}
\mathbb{E}_{n}
\left[
\left\langle
\nabla_{W_\ell}\ell_n,
\Delta W_{\ell,i}^{\mathrm{drop}}
\right\rangle^2
\right]
\]
where $\ell_n$ is the calibration loss of the $n$-th sample. Since the sign of
$\Delta W_{\ell,i}^{\mathrm{drop}}$ is removed by the square, we obtain:
\[
\Delta \mathcal{L}_{\ell,i}
\approx
\frac{1}{2}
\sigma_{\ell,i}^{2}
\mathbb{E}_{n}
\left[
\left(
u_{\ell,i}^{\top}
S_{x,\ell}^{-T}
\nabla_{W_\ell}\ell_n
S_{g,\ell}^{-T}
v_{\ell,i}
\right)^2
\right]
\]
Equivalently, by defining the whitened-space gradient
\[
\nabla_{\widetilde W_\ell}\ell_n
=
S_{x,\ell}^{-T}
\nabla_{W_\ell}\ell_n
S_{g,\ell}^{-T}
\]
the score becomes
\[
s_{\ell,i}
=
\frac{1}{2}
\sigma_{\ell,i}^{2}
\mathbb{E}_{n}
\left[
\left(
u_{\ell,i}^{\top}
\nabla_{\widetilde W_\ell}\ell_n
v_{\ell,i}
\right)^2
\right]
\]
In practice, we estimate the expectation with the calibration set:
\[
s_{\ell,i}
=
\frac{1}{2N}
\sigma_{\ell,i}^{2}
\sum_{n=1}^{N}
\left(
u_{\ell,i}^{\top}
\nabla_{\widetilde W_\ell}\ell_n
v_{\ell,i}
\right)^2
\]
This score measures the empirical-Fisher loss increase of removing each
loss-aware singular component. The constant factor $\frac{1}{2}$ is shared by
all components and does not affect the global ranking, so we omit it and obtain
the importance score used in the main text:
\[
I_{\ell,i}
=
\frac{1}{N}
\sigma_{\ell,i}^{2}
\sum_{n=1}^{N}
\left(
u_{\ell,i}^{\top}
\nabla_{\widetilde W_\ell}\ell_n
v_{\ell,i}
\right)^2
\]
Compared with directly ranking by the K-FAC-predicted local value
$\sigma_{\ell,i}^2$, this score keeps the loss-aware SVD basis induced by
$S_{x,\ell}$ and $S_{g,\ell}$ while calibrating each component on a shared
empirical loss scale across layers and heads.

\subsection{Results}
\label{sec:appx-results}

\subsubsection{More Results in FP16}
\label{sec:appx-results-fp16}

\begin{table*}[h]
    \newcommand{\bc}{\cellcolor{blue!10}}
    \newcommand{\accvalours}[1]{\fontsize{9.75pt}{10pt}\selectfont \textbf{#1}}
    \centering
    \small
    \caption{Accuracy evaluation of different methods under FP16.}
    %\vspace{-5pt}
    \resizebox{\linewidth}{!}{ %
    \begin{tabular}{l|l|c|c|c|c|c|c|c|c|c|c|c}
    \toprule
    
    \multirow{2}{*}{Acc.} & \multirow{2}{*}{Method} & \multicolumn{5}{c|}{ScienceQA-IMG $\uparrow$} & \multicolumn{5}{c|}{SEED-Bench $\uparrow$} & \multirow{2}{*}{Avg. $\uparrow$ } \\
    \cmidrule{3-12}
    & & $\rho_1: 90\%$ & $\rho_1: 80\%$ & $\rho_1: 70\%$ & $\rho_1: 60\%$ & $\rho_1: 50\%$ 
    & $\rho_1: 90\%$ & $\rho_1: 80\%$ & $\rho_1: 70\%$ & $\rho_1: 60\%$ & $\rho_1: 50\%$ \\
    
    \midrule

    \multirow{5.75}{*}{\rotatebox{90}{\makecell{SmolVLM\\2B}}}

    & SVD-LLM   
        & \accval{40.06\%} & \accval{17.20\%} & \accval{3.82\%} & \accval{0.64\%} & \accval{0.69\%}
        & \accval{32.49\%} & \accval{15.89\%} & \accval{4.60\%} & \accval{3.56\%} & \accval{1.23\%}
        & \accval{12.02\%}\\
    & QSVD-noQ   
        & \accval{77.00\%} & \accval{62.77\%} & \accval{42.59\%} & \accval{9.87\%} & \accval{0.20\%}
        & \accval{64.80\%} & \accval{50.46\%} & \accval{36.24\%} & \accval{3.60\%} & \accval{2.25\%}
        & \accval{34.98\%}\\
    & WSVD-noQ   
        & \accval{76.30\%} & \accval{71.74\%} & \accval{60.93\%} & \accval{39.51\%} & \accval{27.27\%}
        & \accval{65.78\%} & \accval{63.29\%} & \accval{54.45\%} & \accval{29.22\%} & \accval{27.35\%}
        & \accval{51.58\%}\\
    & \bc\textbf{LSVD-noQ} 
        & \bc\accval{\textbf{84.58\%}} & \bc\accval{\textbf{83.69\%}} & \bc\accval{\textbf{82.00\%}} & \bc\accval{\textbf{81.11\%}} & \bc\accval{\textbf{78.14\%}}
        & \bc\accval{\textbf{68.54\%}} & \bc\accval{\textbf{68.10\%}} & \bc\accval{\textbf{67.62\%}} & \bc\accval{\textbf{67.12\%}} & \bc\accval{\textbf{66.17\%}}
        & \bc\accval{\textbf{74.51\%}}\\
    \cline{3-12}
    \rule{0pt}{2.75ex}
    & \textcolor{gray}{FP16} & \multicolumn{5}{c|}{\textcolor{gray}{Accuracy: 84.58\%}} & \multicolumn{5}{c|}{\textcolor{gray}{Accuracy: 68.47\%}} & \textcolor{gray}{\accval{76.53\%}}\\

    \midrule
    
    \multirow{5.75}{*}{\rotatebox{90}{\makecell{Qwen2-VL\\7B}}}

    & SVD-LLM   
        & \accval{82.80\%} & \accval{81.11\%} & \accval{80.32\%} & \accval{77.34\%} & \accval{69.31\%}
        & \accval{74.74\%} & \accval{74.27\%} & \accval{71.85\%} & \accval{60.83\%} & \accval{44.36\%}
        & \accval{71.69\%}\\
    & QSVD-noQ   
        & \accval{84.09\%} & \accval{83.34\%} & \accval{82.40\%} & \accval{80.57\%} & \accval{\textbf{76.38\%}}
        & \accval{75.06\%} & \accval{74.97\%} & \accval{72.75\%} & \accval{70.63\%} & \accval{65.30\%}
        & \accval{76.55\%}\\
    & WSVD-noQ   
        & \accval{82.99\%} & \accval{81.90\%} & \accval{67.38\%} & \accval{49.68\%} & \accval{15.87\%}
        & \accval{75.74\%} & \accval{75.27\%} & \accval{74.53\%} & \accval{72.96\%} & \accval{62.16\%}
        & \accval{65.85\%}\\
    & \bc\textbf{LSVD-noQ} 
        & \bc\accval{\textbf{85.23\%}} & \bc\accval{\textbf{84.28\%}} & \bc\accval{\textbf{83.74\%}} & \bc\accval{\textbf{82.00\%}} & \bc\accval{75.41\%}
        & \bc\accval{\textbf{76.22\%}} & \bc\accval{\textbf{76.24\%}} & \bc\accval{\textbf{75.96\%}} & \bc\accval{\textbf{75.14\%}} & \bc\accval{\textbf{69.95\%}}
        & \bc\accval{\textbf{78.42\%}}\\
    \cline{3-12}
    \rule{0pt}{2.75ex}
    & \textcolor{gray}{FP16} & \multicolumn{5}{c|}{\textcolor{gray}{Accuracy: 84.38\%}} & \multicolumn{5}{c|}{\textcolor{gray}{Accuracy: 76.29\%}} & \textcolor{gray}{\accval{80.33\%}}\\

    \midrule
    
    \multirow{5.75}{*}{\rotatebox{90}{\makecell{LLaVA-v1.5\\7B}}}   

    & SVD-LLM  
        & \accval{65.44\%} & \accval{63.71\%} & \accval{61.92\%} & \accval{57.41\%} & \accval{55.53\%}
        & \accval{57.89\%} & \accval{57.50\%} & \accval{55.33\%} & \accval{54.64\%} & \accval{55.31\%} 
        & \accval{58.47\%}\\
    & QSVD-noQ   
        & \accval{67.72\%} & \accval{\textbf{68.22\%}} & \accval{67.08\%} & \accval{65.05\%} & \accval{62.37\%} 
        & \accval{59.84\%} & \accval{59.07\%} & \accval{59.78\%} & \accval{59.00\%} & \accval{58.23\%} 
        & \accval{62.64\%}\\
    & WSVD-noQ   
        & \accval{\textbf{68.17\%}} & \accval{67.72\%} & \accval{\textbf{67.28\%}} & \accval{65.89\%} & \accval{65.49\%}
        & \accval{60.10\%} & \accval{\textbf{60.17\%}} & \accval{59.89\%} & \accval{\textbf{60.18\%}} & \accval{60.46\%}
        & \accval{63.54\%}\\
    % & \rowcolor{blue!10}\textbf{WSVD-noQ} 
    & \bc\textbf{LSVD-noQ}
        & \bc\accval{68.12\%} & \bc\accval{68.02\%} & \bc\accval{67.03\%} & \bc\accval{\textbf{67.08\%}} & \bc\accval{\textbf{66.73\%}}
        & \bc\accval{\textbf{60.14\%}} & \bc\accval{60.00\%} & \bc\accval{\textbf{60.55\%}} & \bc\accval{59.82\%} & \bc\accval{\textbf{60.60\%}}
        & \bc\accval{\textbf{63.81\%}}\\
    \cline{3-12}
    \rule{0pt}{2.75ex}
    & \textcolor{gray}{FP16} & \multicolumn{5}{c|}{\textcolor{gray}{Accuracy: 68.01\%}} & \multicolumn{5}{c|}{\textcolor{gray}{Accuracy: 60.18\%}} & \textcolor{gray}{\accval{64.10\%}}  \\

    \midrule

    % \multirow{5.75}{*}{\rotatebox{90}{\makecell{LLaVA-v1.5\\13B}}}
    
    % & SVD-LLM
    % & \accval{71.05\%} & \accval{70.85\%} & \accval{70.30\%} & \accval{70.35\%} & \accval{70.30\%}
    % & \accval{62.28\%} & \accval{62.34\%} & \accval{62.25\%} & \accval{62.08\%} & \accval{\textbf{63.01\%}}
    % & \accval{66.48\%}\\
    
    % & QSVD-noQ
    % & \accval{71.89\%} & \accval{\textbf{71.99\%}} & \accval{71.49\%} & \accval{71.54\%} & \accval{71.39\%}
    % & \accval{\textbf{62.61\%}} & \accval{62.64\%} & \accval{\textbf{62.82\%}} & \accval{\textbf{62.63\%}} & \accval{62.52\%}
    % & \accval{67.15\%}\\
    
    % & WSVD-noQ
    % & \accval{\textbf{71.99\%}} & \accval{71.84\%} & \accval{\textbf{72.53\%}} & \accval{71.59\%} & \accval{71.44\%}
    % & \accval{62.52\%} & \accval{\textbf{62.68\%}} & \accval{62.38\%} & \accval{62.37\%} & \accval{62.37\%}
    % & \accval{67.17\%}\\
    
    % & \bc\textbf{LSVD-noQ}
    % & \bc\accval{71.89\%} & \bc\accval{71.69\%} & \bc\accval{72.24\%} & \bc\accval{\textbf{72.53\%}} & \bc\accval{\textbf{71.64\%}}
    % & \bc\accval{62.49\%} & \bc\accval{62.50\%} & \bc\accval{62.56\%} & \bc\accval{62.54\%} & \bc\accval{62.34\%}
    % & \bc\accval{\textbf{67.24\%}}\\
    
    % \cline{3-12}
    % \rule{0pt}{2.75ex}
    % & \textcolor{gray}{FP16}
    % & \multicolumn{5}{c|}{\textcolor{gray}{Accuracy: 71.83\%}}
    % & \multicolumn{5}{c|}{\textcolor{gray}{Accuracy: 62.53\%}}
    % & \textcolor{gray}{\accval{67.18\%}}\\
    
    % \midrule
    
    \multirow{5.75}{*}{\rotatebox{90}{\makecell{LLaVA-Next\\7B}}}
    
    & SVD-LLM
    & \accval{68.27\%} & \accval{67.92\%} & \accval{66.58\%} & \accval{66.39\%} & \accval{65.54\%}
    & \accval{68.50\%} & \accval{68.31\%} & \accval{67.65\%} & \accval{67.45\%} & \accval{66.28\%}
    & \accval{67.29\%}\\
    
    & QSVD-noQ
    & \accval{\textbf{70.10\%}} & \accval{69.16\%} & \accval{69.01\%} & \accval{68.27\%} & \accval{66.19\%}
    & \accval{68.86\%} & \accval{68.95\%} & \accval{68.44\%} & \accval{67.98\%} & \accval{67.27\%}
    & \accval{68.42\%}\\
    
    & WSVD-noQ
    & \accval{69.81\%} & \accval{69.56\%} & \accval{69.36\%} & \accval{68.22\%} & \accval{67.87\%}
    & \accval{69.18\%} & \accval{69.27\%} & \accval{\textbf{69.15\%}} & \accval{69.16\%} & \accval{68.59\%}
    & \accval{69.02\%}\\
    
    & \bc\textbf{LSVD-noQ}
    & \bc\accval{69.51\%} & \bc\accval{\textbf{69.81\%}} & \bc\accval{\textbf{69.86\%}} & \bc\accval{\textbf{69.86\%}} & \bc\accval{\textbf{70.80\%}}
    & \bc\accval{\textbf{69.25\%}} & \bc\accval{\textbf{69.29\%}} & \bc\accval{69.13\%} & \bc\accval{\textbf{69.34\%}} & \bc\accval{\textbf{68.77\%}}
    & \bc\accval{\textbf{69.45\%}}\\
    
    \cline{3-12}
    \rule{0pt}{2.75ex}
    & \textcolor{gray}{FP16}
    & \multicolumn{5}{c|}{\textcolor{gray}{Accuracy: 69.60\%}}
    & \multicolumn{5}{c|}{\textcolor{gray}{Accuracy: 69.02\%}}
    & \textcolor{gray}{\accval{69.31\%}} \\
    \midrule
    
    \multirow{5.75}{*}{\rotatebox{90}{\makecell{LLaVA-Next\\13B}}} %LLaVA-v1.5-13B

    & SVD-LLM    
        & \accval{72.53\%} & \accval{72.24\%} & \accval{71.74\%} & \accval{71.15\%} & \accval{70.55\%} 
        & \accval{70.76\%} & \accval{70.63\%} & \accval{70.25\%} & \accval{69.96\%} & \accval{69.58\%} 
        & \accval{70.94\%}\\
    & QSVD-noQ   
        & \accval{71.94\%} & \accval{72.14\%} & \accval{71.74\%} & \accval{72.14\%} & \accval{71.79\%} 
        & \accval{71.23\%} & \accval{71.02\%} & \accval{71.06\%} & \accval{70.92\%} & \accval{70.40\%} 
        & \accval{71.44\%}\\
    & WSVD-noQ   
        & \accval{72.88\%} & \accval{72.98\%} & \accval{\textbf{73.57\%}} & \accval{\textbf{73.48\%}} & \accval{73.28\%}
        & \accval{71.29\%} & \accval{71.17\%} & \accval{71.25\%} & \accval{70.95\%} & \accval{70.81\%}
        & \accval{72.17\%}\\
    & \bc\textbf{LSVD-noQ} 
        & \bc\accval{\textbf{73.13\%}} & \bc\accval{\textbf{73.03\%}} & \bc\accval{73.38\%} & \bc\accval{73.13\%} & \bc\accval{\textbf{73.67\%}}
        & \bc\accval{\textbf{71.39\%}} & \bc\accval{\textbf{71.40\%}} & \bc\accval{\textbf{71.33\%}} & \bc\accval{\textbf{71.37\%}} & \bc\accval{\textbf{71.25\%}}
        & \bc\accval{\textbf{72.31\%}}\\
    \cline{3-12}
    \rule{0pt}{2.75ex}
    & \textcolor{gray}{FP16} & \multicolumn{5}{c|}{\textcolor{gray}{Accuracy: 73.23\%}} & \multicolumn{5}{c|}{\textcolor{gray}{Accuracy: 71.30\%}} & \textcolor{gray}{\accval{72.27\%}}\\

    \bottomrule
    \end{tabular}
    }
    \label{tab:appx-svd-result}
    %\vspace{-10pt}
\end{table*}

The FP16 results in Tab.~\ref{tab:appx-svd-result} further support the robustness of LSVD under a wide range
of low-rank compression ratios. Compared with SVD-LLM, QSVD-noQ, and WSVD-noQ,
LSVD-noQ consistently preserves higher accuracy as the parameter-size ratio
$\rho_1$ decreases. The advantage is especially clear under aggressive
compression. For example, on SmolVLM-2B, several baselines degrade sharply at
$\rho_1=50\%$, while LSVD-noQ still maintains strong performance on both
ScienceQA-IMG and SEED-Bench. Similar trends are observed on Qwen2-VL,
LLaVA-v1.5, and LLaVA-Next models, showing that the proposed loss-aware
decomposition is not specific to a single architecture.

\begin{table}[h]
%\vspace{-10pt}
\newcommand{\bc}{\cellcolor{blue!10}}
    \newcommand{\accvalours}[1]{\fontsize{9.75pt}{10pt}\selectfont \textbf{#1}}
\centering
\small
\caption{Accuracy evaluation of different methods on MMBench under FP16 low-rank compression.}
\resizebox{0.6\textwidth}{!}{%
\begin{tabular}{c|c|ccccc|c}
\toprule
\multirow{2}{*}{Acc.} & \multirow{2}{*}{Method}
& \multicolumn{5}{c|}{MMBench EN V1.1 $\uparrow$}
& \multirow{2}{*}{Avg. $\uparrow$} \\
\cmidrule(lr){3-7}
& & $\rho_1$: 90\% & $\rho_1$: 80\% & $\rho_1$: 70\% & $\rho_1$: 60\% & $\rho_1$: 50\% & \\
\midrule

\multirow{5.75}{*}{\rotatebox{90}{\makecell{SmolVLM\\2B}}}
& SVD-LLM
& \accval{44.31\%} & \accval{33.04\%} & \accval{15.74\%} & \accval{0.73\%} & \accval{0.31\%}
& \accval{18.83\%}\\
& QSVD-noQ
& \accval{58.39\%} & \accval{52.62\%} & \accval{35.79\%} & \accval{3.38\%} & \accval{0.88\%}
& \accval{30.21\%}\\
& WSVD-noQ
& \accval{59.06\%} & \accval{55.90\%} & \accval{51.22\%} & \accval{43.64\%} & \accval{24.52\%}
& \accval{46.87\%}\\
& \bc\textbf{LSVD-noQ}
& \bc\accval{\textbf{64.31\%}} & \bc\accval{\textbf{63.79\%}} & \bc\accval{\textbf{62.96\%}} & \bc\accval{\textbf{61.56\%}} & \bc\accval{\textbf{61.09\%}}
& \bc\accval{\textbf{62.74\%}}\\
\cline{3-8}
\rule{0pt}{2.75ex}
& \textcolor{gray}{FP16}
& \multicolumn{5}{c|}{\textcolor{gray}{Accuracy: 64.47\%}}
& \textcolor{gray}{64.47\%}\\

\midrule

\multirow{5.75}{*}{\rotatebox{90}{\makecell{Qwen2-VL\\7B}}}
& SVD-LLM
& \accval{79.12\%} & \accval{78.29\%} & \accval{73.92\%} & \accval{50.23\%} & \accval{20.83\%}
& \accval{60.48\%}\\
& QSVD-noQ
& \accval{80.16\%} & \accval{79.79\%} & \accval{76.69\%} & \accval{70.06\%} & \accval{64.51\%}
& \accval{74.24\%}\\
& WSVD-noQ
& \accval{80.68\%} & \accval{80.00\%} & \accval{76.31\%} & \accval{67.27\%} & \accval{52.81\%}
& \accval{71.41\%}\\
& \bc\textbf{LSVD-noQ}
& \bc\accval{\textbf{80.68\%}} & \bc\accval{\textbf{80.10\%}} & \bc\accval{\textbf{79.12\%}} & \bc\accval{\textbf{75.43\%}} & \bc\accval{\textbf{66.44\%}}
& \bc\accval{\textbf{76.35\%}}\\
\cline{3-8}
\rule{0pt}{2.75ex}
& \textcolor{gray}{FP16}
& \multicolumn{5}{c|}{\textcolor{gray}{Accuracy: 80.52\%}}
& \textcolor{gray}{80.52\%}\\

\midrule

\multirow{5.75}{*}{\rotatebox{90}{\makecell{LLaVA-v1.5\\7B}}}
& SVD-LLM
& \accval{55.22\%} & \accval{52.52\%} & \accval{47.01\%} & \accval{40.88\%} & \accval{31.06\%}
& \accval{45.34\%}\\
& QSVD-noQ
& \accval{61.25\%} & \accval{60.21\%} & \accval{60.17\%} & \accval{58.60\%} & \accval{56.56\%}
& \accval{59.36\%}\\
& WSVD-noQ
& \accval{60.82\%} & \accval{60.51\%} & \accval{59.94\%} & \accval{58.64\%} & \accval{56.61\%}
& \accval{59.30\%}\\
& \bc\textbf{LSVD-noQ}
& \bc\accval{\textbf{61.51\%}} & \bc\accval{\textbf{61.45\%}} & \bc\accval{\textbf{60.88\%}} & \bc\accval{\textbf{60.00\%}} & \bc\accval{\textbf{57.87\%}}
& \bc\accval{\textbf{60.34\%}}\\
\cline{3-8}
\rule{0pt}{2.75ex}
& \textcolor{gray}{FP16}
& \multicolumn{5}{c|}{\textcolor{gray}{Accuracy: 61.56\%}}
& \textcolor{gray}{\accval{61.56\%}}\\

\midrule

\multirow{5.75}{*}{\rotatebox{90}{\makecell{LLaVA-Next\\7B}}}
& SVD-LLM
& \accval{62.18\%} & \accval{61.04\%} & \accval{58.34\%} & \accval{55.95\%} & \accval{49.82\%}
& \accval{57.47\%}\\
& QSVD-noQ
& \accval{64.83\%} & \accval{64.36\%} & \accval{63.64\%} & \accval{62.08\%} & \accval{59.79\%}
& \accval{62.94\%}\\
& WSVD-noQ
& \accval{64.84\%} & \accval{64.83\%} & \accval{64.04\%} & \accval{64.10\%} & \accval{63.21\%}
& \accval{64.20\%}\\
& \bc\textbf{LSVD-noQ}
& \bc\accval{\textbf{64.99\%}} & \bc\accval{\textbf{65.09\%}} & \bc\accval{\textbf{64.88\%}} & \bc\accval{\textbf{64.94\%}} & \bc\accval{\textbf{64.26\%}}
& \bc\accval{\textbf{64.83\%}}\\
\cline{3-8}
\rule{0pt}{2.75ex}
& \textcolor{gray}{FP16}
& \multicolumn{5}{c|}{\textcolor{gray}{Accuracy: 65.14\%}}
& \textcolor{gray}{\accval{65.14\%}}\\

\midrule

\multirow{5.75}{*}{\rotatebox{90}{\makecell{LLaVA-Next\\13B}}}
& SVD-LLM
& \accval{65.09\%} & \accval{64.52\%} & \accval{61.40\%} & \accval{52.78\%} & \accval{44.05\%}
& \accval{57.57\%}\\
& QSVD-noQ
& \accval{67.27\%} & \accval{66.91\%} & \accval{66.86\%} & \accval{66.18\%} & \accval{63.95\%}
& \accval{66.23\%}\\
& WSVD-noQ
& \accval{67.48\%} & \accval{66.98\%} & \accval{66.86\%} & \accval{66.08\%} & \accval{65.14\%}
& \accval{66.51\%}\\
& \bc\textbf{LSVD-noQ}
& \bc\accval{\textbf{67.58\%}} & \bc\accval{\textbf{67.38\%}} & \bc\accval{\textbf{67.12\%}} & \bc\accval{\textbf{66.91\%}} & \bc\accval{\textbf{66.44\%}}
& \bc\accval{\textbf{67.09\%}}\\
\cline{3-8}
\rule{0pt}{2.75ex}
& \textcolor{gray}{FP16}
& \multicolumn{5}{c|}{\textcolor{gray}{Accuracy: 67.53\%}}
& \textcolor{gray}{\accval{67.53\%}}\\

\bottomrule
\end{tabular}%
}
\label{tab:mmb-en-fp16}
%\vspace{-8pt}
\end{table}
The MMBench EN V1.1 results in Tab.~\ref{tab:mmb-en-fp16} show the same trend on a different evaluation
benchmark. LSVD-noQ achieves the best average accuracy among low-rank methods
for all evaluated models. This indicates that the gains of LSVD are not limited
to ScienceQA-IMG or SEED-Bench, but also transfer to a more general multimodal
reasoning benchmark. In particular, LSVD-noQ remains close to the FP16 baseline
under mild compression and degrades more gracefully than prior SVD-based
methods under stronger compression.

\subsubsection{More Results in Quantized Settings}
\label{sec:appx-results-quant}
The low-precision results in Tab.~\ref{tab:appx-svd-result-quant} demonstrate that LSVD remains effective when
low-rank compression is combined with quantization. LSVD achieves the
highest average accuracy among all compressed low-precision methods, improving
over both quantization-only and SVD-based baselines. Compared with WSVD, LSVD
uses a more loss-aware decomposition, which better
preserves accuracy after quantization. The remaining gap to FP16 is small,
suggesting that LSVD provides a favorable trade-off between compression,
low-precision execution, and accuracy.

\begin{table}[h]
%\vspace{-12pt}
\centering
\small
\caption{Accuracy evaluation of different methods under low-precision on LLaVA-v1.5 and LLaVA-Next models.}
\label{tab:appx-svd-result-quant}
\resizebox{\textwidth}{!}{%
\begin{tabular}{l|ccc|ccc|ccc|c}
\toprule
\multirow{2}{*}{Method}
& \multicolumn{3}{c|}{ScienceQA-IMG $\uparrow$}
& \multicolumn{3}{c|}{SEED-Bench-IMG $\uparrow$}
& \multicolumn{3}{c|}{MMBench EN V1.1 $\uparrow$}
& \multirow{2}{*}{Avg. $\uparrow$} \\
\cmidrule(lr){2-4}
\cmidrule(lr){5-7}
\cmidrule(lr){8-10}
& v1.5 7B & Next 7B & Next 13B
& v1.5 7B & Next 7B & Next 13B
& v1.5 7B & Next 7B & Next 13B
& \\
\midrule
DuQuant
& 57.36\% & 66.34\% & 70.20\%
& 54.11\% & 63.64\% & 66.15\%
& 55.91\% & \textbf{62.78\%} & 64.31\%
& 62.31\% \\

QVLM
& 55.24\% & 60.60\% & 65.28\%
& 50.13\% & 50.38\% & 65.39\%
& 57.15\% & 60.29\% & 63.06\%
& 58.61\% \\

QSVD
& 65.61\% & 66.10\% & 70.43\%
& 58.49\% & 65.63\% & 69.21\%
& 45.45\% & 50.96\% & 45.56\%
& 59.72\% \\

WSVD
& 64.25\% & 66.94\% & \textbf{73.08\%}
& 60.23\% & 67.49\% & 70.67\%
& 56.57\% & 60.56\% & 64.68\%
& 64.94\% \\

\rowcolor{blue!10}
\textbf{LSVD}
& \textbf{66.14\%} & \textbf{68.37\%} & 72.11\%
& \textbf{61.94\%} & \textbf{67.49\%} & \textbf{71.56\%}
& \textbf{58.66\%} & 61.61\% & \textbf{65.08\%}
& \textbf{65.88\%} \\

\midrule
\textcolor{gray}{FP16}
& \textcolor{gray}{68.10\%} & \textcolor{gray}{69.60\%} & \textcolor{gray}{73.23\%}
& \textcolor{gray}{60.18\%} & \textcolor{gray}{69.02\%} & \textcolor{gray}{71.30\%}
& \textcolor{gray}{61.56\%} & \textcolor{gray}{65.14\%} & \textcolor{gray}{67.53\%}
& \textcolor{gray}{67.30\%} \\
\bottomrule
\end{tabular}%
}
%\vspace{-12pt}
\end{table}

% \newpage
\subsubsection{More Results for Ablation Studies}
\label{sec:appx-results-ablation}

\paragraph{Effectiveness of Quantization-aware Whitening}

\begin{wraptable}{r}{0.48\textwidth}
\vspace{-20pt}
\centering
\small
\caption{Results of quantization-aware whitening ablation.}
%\vspace{-5pt}
\resizebox{\linewidth}{!}{ %
\begin{tabular}{c|c|c|c|c}
\toprule
\multirow{2}{*}{Method} & \multicolumn{3}{c|}{ScienceQA-IMG $\uparrow$} & \multirow{2}{*}{Avg. $\uparrow$} \\
\cmidrule{2-4}
& v1.5 7B & Next 7B & Next 13B & \\
\midrule
\textcolor{gray}{FP16}
& \textcolor{gray}{\accval{68.10\%}} 
& \textcolor{gray}{\accval{69.60\%}} 
& \textcolor{gray}{\accval{73.23\%}} 
& \textcolor{gray}{\accval{70.31\%}} \\

W/o QAW
& \accval{65.70\%} 
& \accval{68.07\%} 
& \accval{71.79\%} 
& \accval{68.52\%} \\

LSVD
& \accval{\textbf{66.14\%}} 
& \accval{\textbf{68.34\%}} 
& \accval{\textbf{72.11\%}} 
& \accval{\textbf{68.86\%}} \\
\bottomrule
\end{tabular}
}
\label{tab:quantization-aware-whitening-ablation}
%\vspace{-20pt}
\end{wraptable}

We evaluate quantization-aware whitening under the same W8A8 setting and compression ratio. The \textit{W/o QAW} baseline removes this component while keeping the rest of LSVD unchanged. As shown in Tab.~\ref{tab:quantization-aware-whitening-ablation}, QAW consistently improves accuracy across all models, raising the average from $68.52\%$ to $68.86\%$. This indicates that aligning the SVD transformation with the subsequent quantized execution helps preserve accuracy under low-precision compression.

%We evaluate the effectiveness of quantization-aware whitening under the same W8A8 setting and compression ratio. The \textit{W/o QAW} baseline removes quantization-aware whitening while keeping the remaining LSVD pipeline unchanged. As shown in Tab.~\ref{tab:quantization-aware-whitening-ablation}, quantization-aware whitening consistently improves accuracy across all evaluated models, increasing the average accuracy from 69.07\% to 69.51\%. This indicates that aligning the SVD transformation with the subsequent quantized execution helps preserve accuracy under low-precision compression.

\paragraph{Effectiveness of SVD-aware Permutation}

\begin{wraptable}{r}{0.48\textwidth}
% \vspace{-5pt}
\centering
\small
\caption{Results of SVD-aware permutation ablation.}
%\vspace{-5pt}
\resizebox{\linewidth}{!}{ %
\begin{tabular}{c|c|c|c|c}
\toprule
\multirow{2}{*}{Method} & \multicolumn{3}{c|}{ScienceQA-IMG $\uparrow$} & \multirow{2}{*}{Avg. $\uparrow$} \\
\cmidrule{2-4}
& v1.5 7B & Next 7B & Next 13B & \\
\midrule
\textcolor{gray}{FP16}
& \textcolor{gray}{\accval{68.10\%}} 
& \textcolor{gray}{\accval{69.60\%}} 
& \textcolor{gray}{\accval{73.23\%}} 
& \textcolor{gray}{\accval{70.31\%}} \\

W/o Perm.
& \accval{58.80\%} 
& \accval{50.37\%} 
& \accval{67.92\%} 
& \accval{59.03\%} \\

LSVD
& \accval{\textbf{66.14\%}} 
& \accval{\textbf{68.34\%}} 
& \accval{\textbf{72.11\%}} 
& \accval{\textbf{68.86\%}} \\
\bottomrule
\end{tabular}
}
\label{tab:svd-aware-permutation-ablation}
%\vspace{-5pt}
\end{wraptable}

We evaluate the effectiveness of the SVD-aware permutation in the FFN (Sec.~\ref{sec:mix-quant-svd}) under the same setting as above. The \textit{W/o Perm.} baseline uses the same FFN SVD ratio as LSVD but does not group SVD-friendly channels before factorization. As shown in Tab.~\ref{tab:svd-aware-permutation-ablation}, removing this permutation reduces the average accuracy from $68.86\%$ to $59.03\%$, with a particularly large drop on LLaVA-Next 7B from $68.34\%$ to $50.37\%$. This shows that SVD-aware permutation effectively groups FFN channels that are more suitable for low-rank factorization, thereby preserving accuracy under FFN compression.

%%%%%%%%%%%%%%%%%%%%%%%%%%%%%%%%%%%%%%%%%%%%%%%%%%%%%%%%%%%%

\newpage

\section*{NeurIPS Paper Checklist}

\begin{enumerate}

\item {\bf Claims}
    \item[] Question: Do the main claims made in the abstract and introduction accurately reflect the paper's contributions and scope?
    \item[] Answer: \answerYes{} % Replace by \answerYes{}, \answerNo{}, or \answerNA{}.
    \item[] Justification: The abstract and introduction state the paper's contributions and scope as a novel loss-aware low-rank approximation framework for efficient VLM inference. The claimed accuracy and efficiency improvements are supported by the theoretical analysis, main experiments, and ablation studies.

\item {\bf Limitations}
    \item[] Question: Does the paper discuss the limitations of the work performed by the authors?
    \item[] Answer: \answerYes{} % Replace by \answerYes{}, \answerNo{}, or \answerNA{}.
    \item[] Justification: We discuss the limitation in the last section of the paper.

\item {\bf Theory assumptions and proofs}
    \item[] Question: For each theoretical result, does the paper provide the full set of assumptions and a complete (and correct) proof?
    \item[] Answer: \answerYes{} % Replace by \answerYes{}, \answerNo{}, or \answerNA{}.
    \item[] Justification:  The paper states the assumptions used in the loss-aware lo-rank approximation framework and provides the corresponding derivations and proofs in the main content and appendix.

    \item {\bf Experimental result reproducibility}
    \item[] Question: Does the paper fully disclose all the information needed to reproduce the main experimental results of the paper to the extent that it affects the main claims and/or conclusions of the paper (regardless of whether the code and data are provided or not)?
    \item[] Answer: \answerYes{} % Replace by \answerYes{}, \answerNo{}, or \answerNA{}.
    \item[] Justification: Yes, the paper discloses sufficient information to reproduce its main results. The methods, implementation details, and evaluation setup, including models, datasets, and calibration procedures, are well-documented.

\item {\bf Open access to data and code}
    \item[] Question: Does the paper provide open access to the data and code, with sufficient instructions to faithfully reproduce the main experimental results, as described in supplemental material?
    \item[] Answer: \answerYes{} % Replace by \answerYes{}, \answerNo{}, or \answerNA{}.
    \item[] Justification:  The experiments use publicly available models and benchmarks, and the paper provides detailed experimental settings and implementation details. The code is not included at submission time but will be released upon acceptance.

\item {\bf Experimental setting/details}
    \item[] Question: Does the paper specify all the training and test details (e.g., data splits, hyperparameters, how they were chosen, type of optimizer) necessary to understand the results?
    \item[] Answer: \answerYes{} % Replace by \answerYes{}, \answerNo{}, or \answerNA{}.
    \item[] Justification: The paper specifies the models, benchmark datasets, calibration settings, compression ratios, rank-allocation settings, quantization settings, and hardware/kernel configurations used in the experiments.

\item {\bf Experiment statistical significance}
    \item[] Question: Does the paper report error bars suitably and correctly defined or other appropriate information about the statistical significance of the experiments?
    \item[] Answer: \answerYes{} % Replace by \answerYes{}, \answerNo{}, or \answerNA{}.
    \item[] Justification: Yes, we report the average results over 5 random seeds in all evaluations. And we follow the open-source VLM evaluation toolkit to report the results.

\item {\bf Experiments compute resources}
    \item[] Question: For each experiment, does the paper provide sufficient information on the computer resources (type of compute workers, memory, time of execution) needed to reproduce the experiments?
    \item[] Answer: \answerYes{}% Replace by \answerYes{}, \answerNo{}, or \answerNA{}.
    \item[] Justification: All the details, type of compute workers, memory, time of execution, are provided in the paper.
    
\item {\bf Code of ethics}
    \item[] Question: Does the research conducted in the paper conform, in every respect, with the NeurIPS Code of Ethics \url{https://neurips.cc/public/EthicsGuidelines}?
    \item[] Answer: \answerYes{} % Replace by \answerYes{}, \answerNo{}, or \answerNA{}.
    \item[] Justification: We have read and acknowledge the NeurIPS Code of Ethics.

\item {\bf Broader impacts}
    \item[] Question: Does the paper discuss both potential positive societal impacts and negative societal impacts of the work performed?
    \item[] Answer: \answerYes{} % Replace by \answerYes{}, \answerNo{}, or \answerNA{}.
    \item[] Justification: The paper discusses positive impacts such as reducing inference cost and improving accessibility of efficient VLM deployment, as well as potential risks that more efficient VLMs may also make misuse easier.
    
\item {\bf Safeguards}
    \item[] Question: Does the paper describe safeguards that have been put in place for responsible release of data or models that have a high risk for misuse (e.g., pre-trained language models, image generators, or scraped datasets)?
    \item[] Answer: \answerNA{} % Replace by \answerYes{}, \answerNo{}, or \answerNA{}.
    \item[] Justification: The paper does not release a new high-risk pretrained model, image generator, or scraped dataset. It proposes a compression method evaluated on existing public models, and any use should follow the original models' licenses and usage restrictions.

\item {\bf Licenses for existing assets}
    \item[] Question: Are the creators or original owners of assets (e.g., code, data, models), used in the paper, properly credited and are the license and terms of use explicitly mentioned and properly respected?
    \item[] Answer: \answerYes{} % Replace by \answerYes{}, \answerNo{}, or \answerNA{}.
    \item[] Justification: The paper cites the existing models, datasets, benchmarks, and codebases used in the experiments, and describes their versions, access sources, and licenses or terms of use where available.

\item {\bf New assets}
    \item[] Question: Are new assets introduced in the paper well documented and is the documentation provided alongside the assets?
    \item[] Answer: \answerNA{} % Replace by \answerYes{}, \answerNo{}, or \answerNA{}.
    \item[] Justification: Our paper does not release new models or datasets.

\item {\bf Crowdsourcing and research with human subjects}
    \item[] Question: For crowdsourcing experiments and research with human subjects, does the paper include the full text of instructions given to participants and screenshots, if applicable, as well as details about compensation (if any)? 
    \item[] Answer: \answerNA{} % Replace by \answerYes{}, \answerNo{}, or \answerNA{}.
    \item[] Justification: The paper does not involve crowdsourcing experiments or research with human subjects. 

\item {\bf Institutional review board (IRB) approvals or equivalent for research with human subjects}
    \item[] Question: Does the paper describe potential risks incurred by study participants, whether such risks were disclosed to the subjects, and whether Institutional Review Board (IRB) approvals (or an equivalent approval/review based on the requirements of your country or institution) were obtained?
    \item[] Answer: \answerNA{} % Replace by \answerYes{}, \answerNo{}, or \answerNA{}.
    \item[] Justification: The paper does not involve human-subject studies, crowdsourcing experiments, or collection of data from participants, so IRB approval or equivalent review is not applicable.

\item {\bf Declaration of LLM usage}
    \item[] Question: Does the paper describe the usage of LLMs if it is an important, original, or non-standard component of the core methods in this research? Note that if the LLM is used only for writing, editing, or formatting purposes and does \emph{not} impact the core methodology, scientific rigor, or originality of the research, declaration is not required.
    %this research? 
    \item[] Answer: \answerNA{} % Replace by \answerYes{}, \answerNo{}, or \answerNA{}.
    \item[] Justification: LLMs were used only for writing, editing, and formatting assistance, and did not affect the core methodology, experiments, scientific rigor, or originality of the research.

\end{enumerate}

\end{document}